\documentclass[pdflatex,sn-mathphys-ay]{sn-jnl}

\usepackage{amsmath, amsthm, amsfonts}
\usepackage{mathtools} 

\usepackage{braket}
\usepackage{textcomp}
\usepackage{xspace}
\usepackage{calc}

\usepackage{graphicx}
\usepackage{tikz}
\usetikzlibrary{matrix, positioning, tikzmark}

\usepackage{epstopdf}
\AppendGraphicsExtensions{.tiff}

\usepackage{array, tabularx, booktabs, multirow, makecell}
\renewcommand{\arraystretch}{1.1}

\newcolumntype{Y}{>{\centering\arraybackslash}X}

\usepackage{subcaption}
\usepackage{float}

\usepackage{xcolor}

\usepackage{enumerate}
\usepackage{stackengine}

\usepackage{listings}

\usepackage[title]{appendix}

\renewcommand{\Vec}[1]{\mathbf{#1}} 

\usepackage{hyperref}

\newcommand{\code}[1]{\texttt{#1}}
\newcommand{\expect}[1]{\mathbb{E}\left[#1\right]}
\theoremstyle{thmstyleone}%
\newtheorem{theorem}{Theorem}
\theoremstyle{thmstyletwo}%
\newtheorem{remark}{Remark}%

\theoremstyle{thmstylethree}%

\begin{document}

\title[QMTL]{Parameter-efficient Quantum Multi-task Learning}

\author*[1,2]{\fnm{Hevish} \sur{Cowlessur}}\email{hcowlessur@student.unimelb.edu.au}

\author[2]{\fnm{Chandra} \sur{Thapa}}\email{chandra.thapa@csiro.au}

\author[1]{\fnm{Tansu} \sur{Alpcan}}\email{tansu.alpcan@unimelb.edu.au}

\author[2]{\fnm{Seyit} \sur{Camtepe}}\email{seyit.camtepe@csiro.au}
\affil[1]{\orgdiv{Department of Electrical and Electronic Engineering}, \orgname{University of Melbourne}, \orgaddress{\city{Parkville}, \postcode{3010}, \state{VIC}, \country{Australia}}}

\affil[2]{\orgname{CSIRO}, \orgaddress{\city{Sydney}, \state{NSW}, \country{Australia}}}


\abstract{
Multi-task learning (MTL) improves generalization and data efficiency by jointly learning related tasks through shared representations. In the widely used hard‑parameter‑sharing setting, a shared backbone is combined with task‑specific prediction heads. However, task‑specific parameters can grow rapidly with the number of tasks. Therefore, designing multi‑task heads that preserve task specialization while improving parameter efficiency remains a key challenge.
In Quantum Machine Learning (QML), variational quantum circuits (VQCs) provide a compact mechanism for mapping classical data to quantum states residing in high‑dimensional Hilbert spaces, enabling expressive representations within constrained parameter budgets. Leveraging this property, we propose a \emph{parameter‑efficient quantum multi-task learning} (QMTL) framework that replaces conventional task‑specific linear heads with a fully quantum prediction head in a hybrid architecture. The model consists of a VQC with a shared, task‑independent quantum encoding stage, followed by lightweight task‑specific \textit{ansatz blocks} enabling localized task adaptation while maintaining compact parameterization.
Under a controlled and capacity-matched formulation where the shared representation dimension grows with the number of tasks, our parameter-scaling analysis demonstrates that a standard classical head exhibits quadratic growth, whereas the proposed quantum head parameter cost scales linearly. 
We evaluate QMTL on three multi-task benchmarks spanning natural language processing, medical imaging, and multimodal sarcasm detection, where we achieve performance comparable to, and in some cases exceeding, classical hard-parameter-sharing baselines while consistently outperforming existing hybrid quantum MTL models with substantially fewer head parameters. 
We further demonstrate QMTL's executability on noisy simulators and real quantum hardware, illustrating its feasibility at the circuit scales considered in this study.
}




\maketitle

\section{Introduction}\label{sec:intro}

Multi-task learning (MTL) has become a standard strategy for improving generalization and data efficiency by learning multiple related tasks jointly and enabling knowledge transfer through shared representations~\citep{caruana_multitask_1997, zhang2021survey, ruder2017overview}. 
One of the earliest and most widely adopted formulations follows a hard-parameter-sharing design, in which a common backbone produces a shared feature representation that is subsequently processed by task-specific prediction heads~\citep{caruana_multitask_1997, ruder2017overview}. This backbone-plus-head structure underpins many practical systems, including large pre-trained language models evaluated on heterogeneous benchmarks~\citep{wang_glue_2018}, as well as medical imaging models that jointly diagnose multiple thoracic conditions from chest radiographs~\citep{irvin_chexpert_2019}.

Despite its simplicity and empirical success, hard-parameter-sharing MTL becomes challenging when tasks are weakly related or competing. Because multiple task losses jointly update the same shared parameters, the optimization process can suffer from gradient conflict, negative transfer, task dominance, or capacity competition~\citep{zhang2024proactive, ma2018modeling, zhang2021survey, yu2020gradientsurgerymultitasklearning, standley2020taskslearnedmultitasklearning, yang2026disentanglingtaskconflictsmultitask}. Task-specific prediction heads provide output-level flexibility, but they do not directly eliminate interference arising in the shared representation, where conflicting gradient signals are still aggregated. A common practical response is therefore to increase task-specific capacity or introduce more selective parameter sharing so that each task can better adapt the shared features to its own objective~\citep{misra2016crossstitchnetworksmultitasklearning, shi2023recon}. However, this introduces an important trade-off: improving task-specific flexibility in classical multi-task architectures often increases the number of trainable parameters and/or architectural complexity, particularly when the shared representation is high-dimensional and the number of tasks is large. This creates a practical need for multi-task architectures that retain task-specific specialization while improving parameter efficiency relative to the hard-parameter-sharing MTL architecture.

A broad literature has proposed advanced strategies to address task interference, including soft parameter sharing, mixture-of-experts architectures, task routing mechanisms, cross-stitch style feature sharing, and gradient-balancing methods~\citep{ruder2017overview, zhang2021survey, misra2016crossstitchnetworksmultitasklearning, strezoski2019tasklearningtaskrouting, yu2020gradientsurgerymultitasklearning}. These approaches improve flexibility by selectively sharing computation or modifying the joint optimization process, making them better suited to weakly related tasks. At the same time, they typically introduce greater architectural complexity and, in many cases, additional trainable parameters. In this work, we therefore focus on the standard hard-parameter-sharing setting, which remains a widely used and well-understood baseline and provides a controlled framework for isolating how architectural changes affect parameter efficiency under multi-task specialization.

Recent advances in quantum machine learning (QML) suggest that variational quantum circuits (VQCs), through quantum superposition, provide a compact mechanism for embedding classical data in high-dimensional Hilbert spaces using relatively few trainable parameters~\citep{peruzzo2014variational, schuld2019quantum, sim2019expressibility}. This has motivated studies at the intersection of QML and MTL. 
For example, a hybrid quantum neural network (HQNN) has been proposed in multi-task settings~\citep{phukan_hybrid_2024}. However, they typically preserve conventional classical prediction layers on top of the quantum representations. 
As a result, the parameter growth associated with classical multi-head output structures remains largely unchanged, leaving the potential of quantum circuits to improve head-level parameter efficiency underexplored.

In this work, we propose a hybrid quantum-classical multi-task architecture that replaces conventional task-specific heads with a compact two-stage VQC serving as a fully quantum multi-task head on top of a classical backbone. The proposed design consists of a shared quantum encoding stage that maps the latent representation $\mathbf{Z}$ into a $Q$-qubit quantum state, followed by lightweight task-specific quantum subcircuits and measurement schemes. This structure enables multiple task outputs to be extracted from a small subset of qubits, allowing task-specific processing while limiting the growth of trainable parameters.
Motivated by the broader trade-off in MTL between maintaining shared representations and providing sufficient task-specific flexibility, the proposed architecture combines a compact shared quantum encoder with localized task-specific subcircuits to provide a parameter-efficient form of task-specific specialization without requiring classical output layers.

We analytically show that, under a controlled capacity-matching formulation, the number of trainable parameters in the proposed quantum head grows linearly with the number of tasks. In contrast, under comparable design assumptions, classical hard-parameter-sharing multi-head baselines require substantially more parameters and exhibit quadratic growth as the number of tasks increases. 
Empirical evaluations demonstrate that the proposed approach achieves competitive multi-task performance while using significantly fewer trainable parameters, indicating that improved parameter efficiency does not come at the expense of predictive performance. We further validate the executability of the architecture on noisy simulators and real quantum hardware, supporting its practicality for current-generation quantum devices. Taken together, these results show that the proposed quantum head achieves a favorable performance-parameter efficiency trade-off for hybrid quantum MTL in the near-term setting.

Overall, the key contributions of this work are threefold. First, we introduce a quantum multi‑task head architecture that decomposes task‑independent quantum encoding from lightweight task‑specific quantum specialization within a single variational circuit. Second, we present a structured parameter‑scaling analysis comparing classical hard‑parameter‑sharing heads and a quantum multi‑task head within a controlled, capacity‑matched formulation that explicitly shows how the shared representation dimension scales with the number of tasks. Third, we validate the proposed approach empirically across heterogeneous multi‑task benchmarks and further demonstrate its executability on current noisy quantum hardware. In detail, these contributions are summarized as follows:

\begin{itemize}
\item \textbf{Novel architecture: }We introduce a fully quantum MTL head architecture that separates a shared, task-independent quantum encoder from $T$ lightweight, task-specific subcircuits embedded within a single VQC. Critically, we leverage a qubit-efficient encoding scheme to compactly encode classical features onto a $Q$-qubit quantum circuit. 
Furthermore, by using sets of commuting and non-commuting observables for each task to decouple the qubit count from the output dimensionality, we enable the efficient extraction of rich, task-specific predictions from small quantum sub-registers.
\item \textbf{Analytical framework: }
We formulate the parameter-scaling problem for both classical and proposed QMTL heads and derive closed-form scaling expressions. Under a standard hard-parameter-sharing formulation with independent task-specific heads, and under natural design constraints (Equations \ref{eq:constraint1}--\ref{eq:constraint3}), we prove (Theorem \ref{thm:param_scaling}) that the classical head exhibits quadratic scaling $\mathcal{O}(T^2)$, while the proposed QMTL head scales linearly as $\mathcal{O}(T)$, where $T$ is the number of tasks.

\item \textbf{Empirical analysis:} We evaluate QMTL on three diverse benchmarks (GLUE, CheXpert, and Extended MUStARD) spanning text, images, and multimodal data. Our method matches or exceeds classical baselines while using substantially fewer head parameters (up to $12\times$ reduction) and consistently outperforms existing hybrid quantum-classical (HQNN) baselines. We demonstrate successful implementation on real quantum hardware (IBM Quantum), showing that QMTL retains meaningful performance under realistic noise constraints, validating practical feasibility on near-term quantum devices. We provide detailed ablations on parameter sensitivity and the effect of entanglement in the shared encoder, clarifying which design choices drive performance and efficiency.
\end{itemize}

The remainder of the paper is structured as follows. Section~\ref{sec:related_works} reviews classical and quantum multi-task learning literature, positioning our work in context. Section~\ref{sec:problem} formalizes the problem and research questions we address. Section~\ref{sec:our_model} presents the proposed quantum architecture and provides a rigorous parameter-scaling analysis. Section~\ref{sec:exp_feasibility} describes experimental setups and results across three benchmarks. Section~\ref{sec:ab_studies} presents ablation studies on parameter sensitivity and entanglement. Section~\ref{sec:real_device} reports hardware deployment results. In Section~\ref{sec:discussion}, we summarize the main findings and relate them to the parameter-efficiency analysis. Finally, Section~\ref{sec:concl} presents our overall conclusions, including limitations and future directions.


\section{Related Works}\label{sec:related_works}
In this section, we discuss the literature relevant to classical multi-task learning (MTL) and the intersection of quantum computing and MTL.

\subsection{Classical Multi-Task Learning (MTL)}
In the classical machine learning domain, MTL aims to improve generalization by jointly learning multiple related tasks, allowing information to be shared across tasks during training \citep{caruana_multitask_1997, ruder2017overview, zhang2021survey}. As an inductive transfer learning approach, MTL can improve data efficiency and reduce overfitting compared to training separate models for each task.

A central design question in MTL is how model parameters should be shared across tasks. One of the earliest and most widely adopted approaches is \emph{hard-parameter-sharing} (also known as the feature transformation approach \citep{zhang2021survey}), in which tasks share a common representation while maintaining task-specific output layers. 
However, sharing parameters, as in hard-parameter-sharing MTL, across tasks also introduces important challenges. Because all tasks update the same shared representation, conflicting gradient signals from weakly related tasks can lead to \emph{negative transfer}, degrading performance on some tasks \citep{standley2020taskslearnedmultitasklearning,yu2020gradientsurgerymultitasklearning,sener2019multitasklearningmultiobjectiveoptimization}. In extreme cases, such interference may cause severe performance deterioration on previously learned tasks, a phenomenon related to \emph{catastrophic forgetting} \citep{Kirkpatrick_2017, kar2022preventing}. Furthermore, while task-specific output layers provide partial parameter isolation between tasks, their parameter count increases with the number of tasks and the dimensionality of the shared representation. As the number of tasks or output dimensions grows, these task-specific components can therefore contribute substantially to the overall parameter count.

To address these issues, a broad range of MTL architectures has been proposed. Some approaches focus on improving parameter efficiency by reducing task-specific parameters while maintaining shared representations. However, limiting task-specific capacity can make it more difficult for the model to separate task-specific knowledge, potentially worsening task conflicts \citep{gangwar2025parameterefficientmultitasklearningprogressive, yang2026disentanglingtaskconflictsmultitask}. Other approaches explicitly aim to mitigate task interference through architectural mechanisms such as soft parameter sharing \citep{misra2016crossstitchnetworksmultitasklearning}, task-routing and gating networks \citep{rosenbaum2017routingnetworksadaptiveselection}, mixture-of-experts models \citep{ma2018modeling}, and gradient-balancing techniques \citep{zhang2021survey,ruder2017overview}. While these approaches can reduce task interference, they often introduce additional architectural complexity or task-specific parameters. Consequently, balancing effective task specialization with parameter efficiency remains an open challenge in MTL. Owing to its simplicity, computational efficiency, and strong empirical performance, the hard-parameter-sharing architecture therefore remains widely used as a reference baseline when benchmarking more complex MTL architectures.

While these challenges have been extensively studied in classical MTL, their implications remain largely unexplored in the QML setting. Quantum feature mappings can embed classical data into high-dimensional, expressive Hilbert spaces using compact, parameterized quantum circuits with relatively few trainable parameters \citep{schuld2019quantum, huang2021power, caro2021encoding, ranga2024quantum}. This property suggests that quantum architectures may offer a promising direction for achieving parameter-efficient yet expressive shared representations in multi-task settings. 
Motivated by this gap, we investigate integrating quantum representations into an MTL framework. 
To the best of our knowledge, this work represents an early systematic investigation of parameter efficiency in quantum multi‑task learning with explicit scaling analysis.
We benchmark the proposed approach against the standard hard-parameter-sharing formulation, in which classical task-specific output layers are replaced with a compact quantum circuit. Our design aims to retain sufficient task-specific capacity to potentially limit cross-task interference and mitigate negative transfer in MTL, while substantially reducing the growth of task-specific parameters. Using the hard-parameter-sharing architecture as the reference provides a controlled, well-established baseline, enabling a fair comparison that isolates the effect of quantum components in MTL. 

\subsection{Hybrid and Quantum-Assisted Multi-Task Learning}
A number of works investigate hybrid quantum-classical models in multi-task or multi-output settings, primarily to leverage the high representational power of variational quantum circuits (VQCs) on complex, structured data. For example, \cite{phukan_hybrid_2024} proposed a hybrid architecture for joint analysis of sarcasm, sentiment, and emotion across text, audio, and visual modalities, in which a shared classical encoder feeds into task-specific VQCs. Similar hybrid pipelines attach small quantum circuits to classical backbones for language understanding, multimodal fusion, or time-series prediction, and then use classical linear layers to produce task-specific logits \citep{buonaiuto_multilingual_2025, li2025qmlsc}. In all of these models, the primary motivation for introducing QML is to exploit the expressive feature maps induced by VQCs, while the final prediction layers remain classical.

From the perspective of parameter efficiency in multi-task learning, this creates an important gap. Because the quantum modules are followed by classical linear heads, the overall head parameter count still scales with the number of tasks and the latent dimension in essentially the same way as in purely classical MTL: each task is equipped with its own classical readout, and all substantial parameter savings must come from the backbone, not from the head structure itself. Moreover, existing hybrid QML work rarely analyzes how the number of trainable parameters in the heads scales with the number of tasks. 

By contrast, our architecture uses a shared variational quantum encoder that maps classical features $\mathbf{Z}\in\mathbb{R}^d$ into a high-dimensional Hilbert space, followed by $T$ lightweight quantum subcircuits, all embedded in the \emph{same} circuit and optimized jointly. All trainable parameters in the head are quantum, and we provide an explicit scaling analysis showing that the quantum head scales linearly in $T$, with the per-task cost governed by the number of qubits and layers in the small task-specific subcircuits.

\subsection{Fully Quantum Multi-Task Architectures}

Closer to our setting are recent proposals for fully quantum multi-task learning (QMTL) architectures. Configured quantum reservoirs have been shown to support multi-task prediction of chaotic and biological dynamics on a single quantum device \citep{xia2023configured}. More recently, in the domain of quantum finance, contextual quantum neural networks with a share-and-specify ansatz have been introduced for multi-asset stock price prediction \citep{mourya2026contextual}. In that work, task-specific operators are controlled by quantum labels, enabling several assets to be trained simultaneously on a single circuit, demonstrating a fully quantum QMTL architecture tailored to financial time series.  

These works highlight that quantum coherence and shared quantum representations can be beneficial in multi-task settings. However, they focus on time-series or physical-system modeling, often without an explicit classical backbone, and do not consider classical multi-task benchmarks with heterogeneous tasks or analyze head parameter counts as a function of the latent dimension and the number of tasks. Our model addresses this gap by (i) operating on a fixed classical feature vector $Z$ produced by a backbone/feature extractor, (ii) using a single VQC that separates a task-independent encoding from task-specific subcircuits, and (iii) leverages sets of commuting and non-commuting observables for measurements across tasks, so that the number of logits $r$ is not tied to the number of qubits. We show how this leads directly to the favorable parameter scaling derived in Section~\ref{subsec:param_efficiency}.


\section{Problem Statement and Research Questions}\label{sec:problem}

We consider the standard hard-parameter-sharing MTL setting~\citep{caruana_multitask_1997, ruder2017overview}, where a single feature extractor is shared across all tasks, and task-specific predictions are produced by lightweight heads operating on the shared representation. Let $\mathcal{D}=\{(\mathbf{x}_i,\{\mathbf{y}_{i,t}\}_{t=1}^T)\}_{i=1}^N$ denote a dataset of $N$ input samples, where $\mathbf{x}_i$ is the $i$-th input and $\mathbf{y}_{i,t}$ denotes the target label associated with task $t \in \{1,\ldots,T\}$. A shared backbone network $\Phi_{\text{backbone}}$ maps inputs to a latent representation $\mathbf{Z}_i=\Phi_{\text{backbone}}(\mathbf{x}_i)\in\mathbb{R}^d$. Task-specific prediction heads $h_t:\mathbb{R}^d\to\mathbb{R}^{r}$ produce outputs $\hat{\mathbf{y}}_{i,t}=h_t(\mathbf{Z}_i)$. We assume the task index $t$ is known when forming predictions (as in standard MTL benchmarks); the model can output all $\{\hat{\mathbf{y}}_{i,t}\}_{t=1}^T$ or use the corresponding head $h_t$ for the task of interest.

In MTL, the objective is to jointly minimize a combined loss:
\begin{equation}
\mathcal{L}_{\text{MTL}}=\frac{1}{N}\sum_{i=1}^N\sum_{t=1}^{T}\lambda_t\,\mathcal{L}_t(\hat{\mathbf{y}}_{i,t},\mathbf{y}_{i,t}),
\end{equation}
where $\lambda_t$ are per-task loss weights (typically uniform), and $\mathcal{L}_t$ is the task-specific loss (e.g., cross-entropy for classification, MSE for regression).

In practice, the shared backbone $\Phi_{\text{backbone}}$ is often large and pre-trained (e.g., a convolutional network or a transformer) and is treated as fixed or lightly fine-tuned for downstream tasks. As a result, a substantial portion of the task-dependent design freedom lies in the multi-task head. In standard hard-parameter-sharing MTL, task-specific heads typically instantiate separate mappings from the shared representation $\Vec{Z}\in\mathbb{R}^d$ to task outputs, leading to a head parameter count of ${\sum_{t=1}^T r_t(d+1)}$ for linear heads. When the number of tasks and/or the shared feature dimension $d$ is large, the head can become a non-trivial component of the overall parameter budget. This motivates the central question studied in this work: Can we redesign the multi-task head to improve parameter efficiency without sacrificing predictive performance?

Thus, we frame the central problem addressed in this work as: 
\emph{How can we construct a multi-task head that leverages quantum properties such as superposition,
entanglement, and high-dimensional Hilbert spaces to reduce parameter cost while maintaining
sufficient task-specific capacity and predictive performance across tasks?}

\noindent
This leads to specific sub-questions that allow us to systematically address the main question:
\begin{itemize}
    \item How do we compactly encode the latent feature vector $\Vec{Z}$ into a quantum state that supports task-specific specialization? \textit{(Section~\ref{subsec:our_model})}
    \item How should we decompose the quantum MTL head into a shared encoding stage and task-specific subcircuits to maximize parameter savings? \textit{(Section~\ref{subsec:our_model})}
    \item How does the parameter scaling of the proposed quantum head compare quantitatively to classical baselines as a function of $d$, $T$, and $r$? \textit{(Section~\ref{subsec:param_efficiency})}
    \item How well does the proposed QMTL head transfer across domains and input modalities (text, vision, and multimodal) when paired with different feature extractors and heterogeneous task types? \textit{(Section~\ref{sec:exp_feasibility})}
\end{itemize}

\section{Parameter-efficient Quantum Multi-task Learning \label{sec:our_model}}
In this section, we outline the quantum design principles underlying our architecture and analyze their impact on parameter scaling in multi-task learning. While classical MTL is a mature and active research area, our focus is on isolating the parameter-scaling behavior of standard multi-head formulations within a common, controlled design setting. 

\subsection{Proposed quantum circuit architecture for QMTL\label{subsec:our_model}}

\noindent
\textbf{Design rationale.}\quad
The design of our QMTL head is motivated by three principles: (i) expressive shared quantum encoding, (ii) reuse of a single shared state across tasks, and (iii) measurement-centric readout for compact task outputs. We embed the classical feature vector $\mathbf{Z}$ into an exponentially large ambient Hilbert space (dimension $2^Q$) via quantum superposition on $Q$ qubits, using a parameterized feature-map circuit whose trainable parameters scale linearly with $Q$ (for fixed depth, $D$). 
The resulting quantum state provides an expressive shared representation whose amplitudes encode features common across tasks. 
Multi-task efficiency arises because this shared state is prepared once per input, then specialized by lightweight, task-local subcircuits rather than by large, task-specific parameter blocks. 
Finally, we obtain task outputs via multi-observable readout, measuring a set of Pauli expectation values on each task sub-register. This enables output dimensionality larger than the sub-register size, at the cost of additional shots rather than additional trainable parameters.

Fig.~\ref{fig:our_qmtl_circuit} illustrates the proposed parameter-efficient circuit architecture for multi-task learning. The complete circuit consists of two stages executed within a single VQC: (i) a shared, task-independent encoding stage that maps a backbone feature vector $\mathbf{Z}\in\mathbb{R}^{d}$ to a task-agnostic quantum state $|\psi_\Theta(\mathbf{Z})\rangle$, and (ii) a task-specialization stage that partitions the register into $T$ task-local subcircuits and produces task-specific predictions $\{\mathbf{y_t}\}_{t=1}^T$ via measurement.
We assume the task set $\{1,\dots,T\}$ is known, as in standard supervised multi-task learning, and the circuit outputs $(\hat{y}_1,\dots,\hat{y}_T)$ in a fixed order. We detail the two stages below.

\begin{figure}[h]
    \centering
    \includegraphics[width=\linewidth]{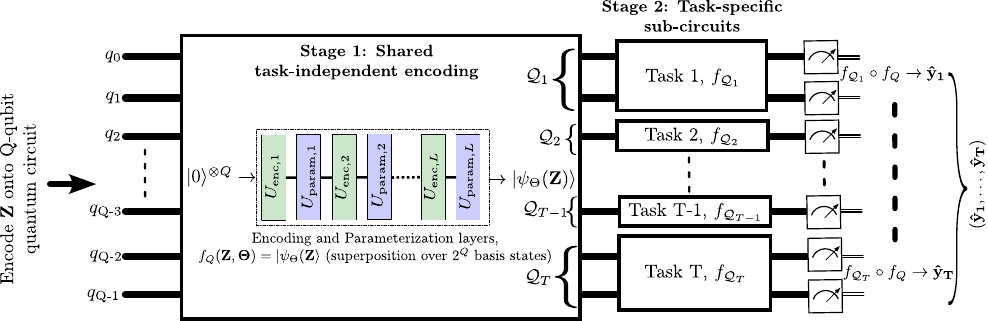}
    \caption{
    Proposed Quantum Circuit architecture for multi-task learning. An extracted feature vector $\textbf{Z}$ is embedded into a $Q$-qubit register via a shared trainable state-preparation stage $f_Q(\mathbf{Z},\mathbf{\Theta})$. This embeds $\textbf{Z}$ to a quantum state $|\psi_\Theta(\mathbf{Z})\rangle$ in superposition over computational basis states (a vector in a $2^Q$-dimensional complex Hilbert space). The register is then partitioned into $T$ disjoint qubit sub-registers, $\mathcal{Q}_t \subseteq \{q_0,\dots, q_{Q-1}\}$, and task-specific parameterized blocks $\{f_{\mathcal{Q}_t}\}_{t=1}^T$ are applied on each sub-register to produce $T$ measurement outcomes $\{\mathbf{\hat{y}_t\}}^T_{t=1}$ within one circuit execution. This enables parameter-efficient MTL by sharing a common trainable backbone while allowing task specialization through localized operations and readouts.}
    \label{fig:our_qmtl_circuit}
\end{figure}

\subsubsection*{Stage 1: Shared task-independent encoding}
For each qubit $q_i \in \{q_0, \ldots, q_{Q-1}\}$, we first apply a Hadamard gate to create an equal superposition:
\begin{equation}
\ket{\psi_0} = H^{\otimes Q} \ket{0}^{\otimes Q} = \frac{1}{\sqrt{2^Q}} \sum_{s=0}^{2^Q-1} \ket{s}.
\end{equation}
Then we apply \(L\) blocks of alternating encoding and trainable operations to embed the feature vector \(\mathbf{Z} \in \mathbb{R}^d\) into a \(Q\)-qubit state \(\ket{\psi_\Theta(\mathbf{Z})}\), with \(Q\!\times\!L=d\). This encoding method is adapted from the qubit-efficient data loading technique proposed in \citep{cowlessur_hybrid_2025}. 
\begin{equation}
\ket{\psi_\Theta(\mathbf{Z})}
=\Big(\prod_{\ell=1}^{L} U_{\text{param},\ell}(\boldsymbol{\theta}_\ell)\,U_{\text{enc},\ell}(\mathbf{Z})\Big)\ket{\psi_0},
\end{equation}
where
\[
U_{\text{enc},\ell}(\mathbf{Z})
=\bigotimes_{i=Q(\ell-1)}^{Q\ell-1} R_x(Z_i),\qquad
U_{\text{param},\ell}(\boldsymbol{\theta}_\ell)
=\Big(\prod_{\text{ladder}} \mathrm{CNOT}\Big)\,
\bigotimes_{i=0}^{Q-1}\!R_y(\theta_{\ell,i})R_z(\theta_{\ell,i}).
\]
This prepares a task-agnostic shared representation on the register,
\(\ket{\psi_\Theta(\mathbf{Z})}\) (or \(\rho_\Theta(\mathbf{Z})=\ket{\psi}\!\bra{\psi}\) in density-operator form), which downstream task heads locally act upon. The state implicitly encodes the shared features $\mathbf{Z}$ in a superposed, high-dimensional representation, $\ket{\psi_\Theta(\mathbf{Z})}$.

\subsubsection*{Stage 2: Task-specific subcircuits and multi-observable readout}
The shared $Q$-qubit register is partitioned into $T$ disjoint qubit subsets, $\mathcal{Q}_t=\{q_{t,1}, \dots, q_{t,S_t}\}$, where $|\mathcal{Q}_t|=S_t$, and $\sum_t S_t=Q$. For each task $t$, we apply a shallow parameterized unitary transformation, $U_t(\phi_t)$. In the circuit diagram (Fig.~\ref{fig:our_qmtl_circuit}), this task-specific block is denoted by $f_{\mathcal{Q}_t}$ and implemented as $f_{\mathcal{Q}_t}\equiv U_t(\boldsymbol{\phi}_t)$. We instantiate $U_t(\phi_t)$ using Pennylane's \code{StronglyEntanglingLayer} \emph{ansatz} \citep{pennylane_2022}, which alternates single-qubit rotations with fixed CNOT entangling gates. The sub-register size $S_t$ is set per task according to the required output dimension, while the number of \emph{ansatz} layers is treated as a tunable depth hyperparameter in our experiments.

To obtain the $r_t$ outputs for task $t$, instead of allocating one qubit per logit, we evaluate multiple Pauli observables on $\mathcal{Q}_t$, arranged into commuting measurement groups. For example, assuming task $t=1$ has $3$-dimensional output logits, we assign $\mathcal{Q}_1 = \{q_{1,1}, q_{1,2}\}$ where $|\mathcal{Q}_1|=S_1=2$ and $r_1=3$ we use:
\begin{equation}
\label{eq:readout}
\mathbf{o}_t \;=\; \big(\,\langle Z_{q_{1,1}}\rangle,\;\langle Z_{q_{1,2}}\rangle,\;\langle X_{q_{1,1}} X_{q_{1,2}}\rangle\,\big),  
\end{equation}
where $\langle\cdot\rangle$ denotes the expectation value with respect to the task-adapted quantum state (i.e., after applying the task-specific \emph{ansatz} $U_1(\phi_1)$), commuting observables can be estimated from the same state preparation and basis setting; non-commuting groups require separate settings (higher shot cost but no additional trainable parameters).

While expectation values are bounded, downstream loss functions and task‑specific normalization implicitly rescale these outputs; we additionally explore affine calibration where required (e.g., regression tasks). We use these expectations directly as task outputs,
\begin{equation}
\label{eq:logits}
\hat{\mathbf{y}}_t = \mathbf{o}_t = (o_{t,1}, \ldots, o_{t,r_t}),
\end{equation}
and compute task losses from $\hat{\mathbf{y}}_t$ without additional classical
post-processing layers. This design choice isolates the contribution of quantum parameter efficiency from classical post-processing capacity.

In the next subsection, we formalize how these design choices translate into the total trainable parameter count as a function of the number of tasks, and compare against standard classical multi-head baselines.

\subsection{Parameter Efficiency Analysis}
\label{subsec:param_efficiency}
 \begin{figure}[h]
    \centering
    \includegraphics[width=\linewidth]{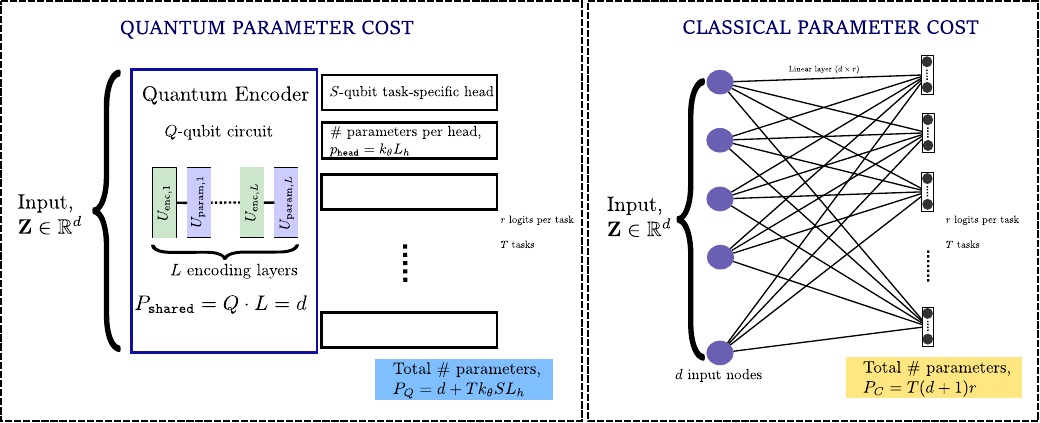}
    \caption{Total number of trainable parameters in our proposed architecture (shown in Fig.~\ref{fig:our_qmtl_circuit}) compared against the simplest classical counterpart. We assume that the output dimension of each task is $r$ to aid analysis in this section.}
    \label{fig:param_comparison_diag}
\end{figure}

We compare the number of trainable parameters in a standard classical multi-task head and in the proposed quantum multi-task head. Throughout, we assume that the backbone network producing the shared latent representation is fixed and identical in both cases, so that only the head parameters differ. This controlled setup isolates the parameter growth attributable to task-specific heads. 

We acknowledge that more expressive multi-task learning approaches exist (e.g., adapter-based methods \citep{pfeiffer2020madxadapterbasedframeworkmultitask}, task-routing mechanisms \citep{strezoski2019tasklearningtaskrouting}, cross-stitch networks \citep{misra2016crossstitchnetworksmultitasklearning}, or mixture-of-experts architectures \citep{ma2018modeling}). 
However, such designs introduce additional task-conditioned modules or routing components to improve flexibility and mitigate task interference. 
Hence, in this work, we focus on the standard hard-parameter-sharing formulation with independent task-specific heads, which serves as a widely used and interpretable baseline. This choice enables a controlled analysis of parameter scaling, allowing us to isolate the effect of replacing classical task-specific heads with the proposed quantum multi-task head.

Let $\Vec{Z} \in \mathbb{R}^d$ denote the latent feature vector shared across tasks, and let $T$ be the number of tasks. For analytical simplicity, we assume that each task has the same output dimensionality $r$, so that each classical task-specific head output and quantum task-specific subcircuit readout are represented by $r$ logits.

\vspace{1em}
\noindent
\textbf{Classical Multi-Task Head Parameters.}\quad
We consider the standard hard-parameter-sharing formulation \citep{caruana_multitask_1997, ruder2017overview, zhang2021survey}, in which a shared backbone produces a task-independent representation $\mathbf{Z} \in \mathbb{R}^d$, followed by independent task-specific prediction heads.

While task heads are often implemented as small multi-layer perceptrons (MLPs), for clarity of analysis, we consider the simplest and most parameter-efficient instantiation: a single linear layer per task,
\begin{equation}
f_t^{(C)}(\mathbf{Z}) = \mathbf{W}_t \mathbf{Z} + \mathbf{b}_t,
\end{equation}
where $\mathbf{W}_t \in \mathbb{R}^{r \times d}$ and $\mathbf{b}_t \in \mathbb{R}^{r}$ are the task-specific parameters for task $t$, and $r$ denotes the output dimension of the task. Any deeper MLP-based head would only increase the parameter count relative to this baseline, and therefore would not alter the scaling comparison presented below. 
Therefore, the parameter count per task is:
\begin{equation}
\label{eq:O(rd)_rel}
p^{(C)}_{\text{head}} = rd + r = r(d + 1),
\end{equation}
and the total classical head parameter count is:
\begin{equation}
P_C = T \times r(d+1) = Tr(d+1).
\end{equation}

\vspace{1em}
\noindent
\textbf{Quantum Multi-Task Head Parameters.}\quad
The quantum head has two components:
\begin{enumerate}
    \item \textbf{Shared encoding stage:} The encoding subcircuit comprises $L$ layers, each with one trainable rotation per qubit and layer (per Section~\ref{subsec:our_model}). Thus
    \begin{equation}
    P_{\text{shared}} = Q \times L.
    \end{equation}
    \item \textbf{Task-specific subcircuits:} Each task $t$ has a shallow subcircuit with depth $L_{t,h}$ and $k_{t,\theta}$ trainable parameters per qubit per layer. For analytical tractability, we assume the task-specific subcircuits have the same configuration, by setting $S_t = S, L_{t,h} = L_{h}$, and $k_{t,\theta} = k_{\theta}$. Therefore, the parameter count in each head is:
    \begin{equation}
    p^{(Q)}_{\text{head}} = S \times L_h \times k_\theta,
    \end{equation}
    
    and the total quantum head parameters:
    \begin{equation}
    P_Q = P_{\text{shared}} + T \times p^{(Q)}_{\text{head}} = QL + T \times k_\theta S L_h.
    \end{equation}
\end{enumerate}

\vspace{1em}
\noindent
\textbf{Design Constraints and Scaling Relations.}\quad
To make the parameter-scaling comparison explicit, we introduce two structural design constraints that relate circuit width, depth, and task partitioning:

\begin{equation}
\boxed{QL = d} \label{eq:constraint1}.
\end{equation}
We adopt this constraint as a capacity‑matching design choice, ensuring that the quantum encoder processes the full $d$-dimensional latent representation under a fixed parameter budget (encoding stage with $L$ layers and $Q$ qubits). Next, we set:

\begin{equation}
\boxed{Q = S \times T} \label{eq:constraint2}.
\end{equation}
This partitions the $Q$ qubits into $T$ disjoint blocks of $S$ qubits, one block per task.

Substituting these into Equation~(\ref{eq:constraint1}):
\begin{equation}
\boxed{d = QL = STL} \label{eq:constraint3}.
\end{equation}

\textit{\textbf{Remarks:}
The scaling gap is structural and arises from how parameters are allocated under capacity-matched designs. In classical hard-parameter-sharing multi-head architectures, each task employs a dense mapping $W_t \in \mathbb{R}^{r \times d}$, resulting in $O(r d)$ parameters per task (Equation~\ref{eq:O(rd)_rel}). 
The resulting quadratic scaling of classical heads does not arise from task multiplicity alone (i.e., the presence of $T$ tasks), but from the need to increase the shared representation dimension $d$ with $T$ under comparable capacity assumptions (Equation~\ref{eq:constraint3}). This also reflects the requirement that the shared representation remains sufficiently expressive as the number and diversity of tasks increase, avoiding representational bottlenecks in the shared encoder. In this regime, where $d = \mathcal{O}(T)$, the total number of head parameters scales as $\mathcal{O}(T \cdot r d) = \mathcal{O}(T^2)$.
In contrast, our architecture constrains parameter growth through a compact shared quantum encoding stage, while each task introduces only a shallow subcircuit acting on a fixed-size subset of qubits. Task outputs are obtained via observable-based readout, which increases measurement cost but does not introduce additional trainable parameters. Under constraints~\ref{eq:constraint1}--\ref{eq:constraint3}, this design yields linear scaling in $T$ for the quantum head, as formalized in Theorem~\ref{thm:param_scaling}.}

\begin{theorem}[Parameter Scaling of Classical vs Quantum Heads]
\label{thm:param_scaling}

Under design constraints~(\ref{eq:constraint1})--(\ref{eq:constraint3}), the parameter counts as functions of task count $T$ are:

\textbf{Classical:}
\begin{equation}
P_C = Tr(d+1) = Tr(STL+1) \approx Tr \times STL = \mathcal{O}(T^2)
\end{equation}

\textbf{Quantum:}
\begin{equation}
P_Q = QL + Tk_\theta SL_h = STL + Tk_\theta SL_h = ST(L + k_\theta L_h) = \mathcal{O}(T)
\end{equation}

The ratio scales as:
\begin{equation}
\frac{P_Q}{P_C} \approx \frac{ST(L + k_\theta L_h)}{STLr \times T} = \frac{L + k_\theta L_h}{Lr} \times \frac{1}{T} \to 0 \quad \text{as } T \to \infty.
\end{equation}
\end{theorem}

\begin{proof}
Substituting constraints~(\ref{eq:constraint1})--(\ref{eq:constraint3}) directly into the parameter count expressions yields the stated scalings. Since $r, L, k_\theta, L_h$ are fixed design hyperparameters independent of $T$, the leading-order behavior in $T$ dominates, establishing the claimed asymptotic scaling orders.
\end{proof}

\begin{remark}
The parameter ratio $\frac{P_Q}{P_C} \propto \frac{1}{T}$ demonstrates \emph{superlinear} efficiency gains as the number of tasks increases. For typical values ($r=2$, $L=3$, $k_\theta=1$, $L_h=1$), we have:
\begin{equation}
\frac{P_Q}{P_C} \approx \frac{4}{6} \times \frac{1}{T} = \frac{2}{3T}.
\end{equation}
Thus, for $T=10$ tasks, $P_Q \approx \frac{P_C}{15}$; for $T=100$, $P_Q \approx \frac{P_C}{150}$.
\end{remark}

\vspace{1em}
\noindent
\textbf{Implicit high-dimensional representation and relations to performance.}\quad

\noindent
The shared quantum encoder (Stage 1) prepares a quantum state
\begin{equation}
\ket{\psi_\Theta(\mathbf{Z})} \in \mathcal{H}^{2^Q},
\end{equation}
which represents the input in a $2^Q$-dimensional Hilbert space. This state is generated implicitly through a parameterized quantum circuit using $d = QL$ trainable parameters. 

Rather than explicitly constructing high-dimensional feature vectors, the quantum encoder provides a compact mechanism for embedding inputs into a shared, high-dimensional representation across tasks. In the context of MTL, this allows multiple task-specific sub-circuits to operate on a common expressive representation without increasing the number of trainable parameters in proportion to the implicit feature dimension.

In our experiments (Section~\ref{sec:exp_feasibility}), we instantiate this architecture with a fixed number of tasks and observe that, despite using substantially fewer parameters than the classical head baseline, the quantum head achieves comparable performance and, in some settings, outperforms it. These findings suggest that the shared quantum representation provides sufficient expressivity while enabling improved parameter efficiency.

\subsection{Training Strategy}

Training proceeds via standard hybrid quantum-classical optimization. We use the parameter-shift rule to compute gradients of quantum circuits:
\begin{equation}
\frac{\partial \expect{O}}{\partial \theta} = \frac{\expect{O(\theta + \frac{\pi}{2})} - \expect{O(\theta - \frac{\pi}{2})}}{2},
\end{equation}
which evaluates the circuit twice per parameter per gradient step. All parameters ($\Theta$ in the encoder and $\{\boldsymbol{\phi}_t\}$ in task-specific heads) are jointly optimized via backpropagation. We assume that task identity is known during training (and inference), consistent with standard supervised multi-task learning.

Training follows a balanced task-interleaved schedule, in which mini-batches from all tasks are sampled throughout training, so that shared parameters receive gradient updates from all tasks within each training epoch. This contrasts with sequential learning \citep{Kirkpatrick_2017}, where the model is trained on one task for multiple consecutive epochs before switching to another, a regime known to induce catastrophic forgetting in shared representations. In this work, the overall loss is computed as the sum (or average) of task-specific losses within each update step, so that gradients from all active tasks jointly update the shared encoder and task-specific subcircuits.  This ensures that the shared quantum parameters are optimized under simultaneous multi-task supervision rather than being dominated by a single task. For datasets with imbalanced label distributions (e.g., CheXpert \citep{CheXpertStanford2019}), we apply task-level loss normalization to prevent tasks with larger sample counts from disproportionately influencing the shared gradients.

The training procedure can be either parallel (all tasks in one mini-batch) or alternating (task-specific batches interleaved across steps), depending on the dataset labeling scheme. For datasets with per-sample labels for all tasks (e.g., CheXpert, Extended MUStARD~\citep{chauhan_sentiment_2020}), parallel joint training is feasible. For partially labeled datasets (e.g., GLUE~\citep{gluebenchmark}), we alternate task-specific updates as described in Appendix~\ref{appendix:methods}. More advanced approaches, such as continual learning, sparse training, task-agnostic inference, dynamic routing, or adaptive loss weighting, are complementary directions but fall outside the scope of this work.


\section{Quantum Multi Task Learning Experiments and Results}\label{sec:exp_feasibility}
In this section, we describe the experiments and present our simulation results when we evaluate the proposed QMTL head on three established multi-task benchmark datasets spanning text, medical imaging, and multimodal dialogue: \textbf{GLUE} \citep{gluebenchmark}, \textbf{CheXpert} \citep{CheXpertStanford2019}, and \textbf{Extended MUStARD} \citep{chauhan_sentiment_2020}. 

\subsection{Experimental setup}
\label{subsec:exp_overview}

We evaluate our proposed QMTL on three diverse benchmark datasets to demonstrate the breadth of applicability:
\begin{enumerate}
    \item \textbf{GLUE} (NLP): nine heterogeneous NLP tasks (binary/multiclass/regression),
    \item \textbf{CheXpert} (medical imaging): five chest X-ray pathology classification tasks, and
    \item \textbf{Extended MUStARD} (multimodal): five tasks spanning sarcasm, sentiment, and emotion recognition.
\end{enumerate}

Across all benchmarks, we extract feature vectors $\mathbf{Z}\in\mathbb{R}^d$ using a backbone similar to that in previous work. These consist of a fixed classical backbone, namely frozen BERT \citep{devlin2019bertpretrainingdeepbidirectional} for GLUE, ImageNet-pretrained DenseNet121 \citep{huang2018denselyconnectedconvolutionalnetworks} for CheXpert, and author-released multimodal pre-extracted features for extended MUStARD \citep{chauhan_sentiment_2020}. 

For each backbone and dataset, we compare QMTL's performance with two multi-task heads that serve as benchmarks. These are:  
\begin{enumerate}
    \item a classical multi-head baseline following the hard-parameter-sharing architecture, and 
    \item a hybrid quantum neural network (HQNN) baseline from \citep{phukan_hybrid_2024}
\end{enumerate}
Detailed dataset and tasks description, dataset preprocessing, metrics, and implementation details are deferred to Appendices~\ref{appendix:datasets}, ~\ref{appendix:models}, and ~\ref{appendix:methods}.
Fig~\ref{fig:exp_pipeline_small} summarizes the overall experimental pipeline.
\begin{figure}[htbp!]
\centering
\scriptsize
\begin{tikzpicture}[
    >=latex,
    node distance=0.5cm and 0.8cm,
    dataset/.style={rectangle, rounded corners, draw, align=center,
                    minimum width=1.9cm, minimum height=0.7cm},
    encoder/.style={rectangle, rounded corners, draw, align=center,
                    minimum width=2.2cm, minimum height=0.7cm, fill=gray!5},
    head/.style={rectangle, rounded corners, draw, align=center,
                 minimum width=2.6cm, minimum height=0.7cm},
    metric/.style={rectangle, rounded corners, draw, align=center,
                   minimum width=2.2cm, minimum height=0.7cm, fill=gray!10}
]

\node[dataset] (glue) {GLUE};
\node[dataset, below=0.8cm of glue] (chexpert) {CheXpert};
\node[dataset, below=0.8cm of chexpert] (mustard) {Ext.\\MUStARD};

\node[encoder, right=of glue] (enc_glue) {Frozen\\BERT};
\node[encoder, right=of chexpert] (enc_chexpert) {DenseNet-121};
\node[encoder, right=of mustard] (enc_mustard) {Text / Multi-\\modal};

\node[head, right=of enc_glue] (head_glue) {Classical / HQNN / QMTL};
\node[head, right=of enc_chexpert] (head_chexpert) {Classical / HQNN / QMTL};
\node[head, right=of enc_mustard] (head_mustard) {Classical / HQNN / QMTL};

\node[metric, right=of head_glue] (metric_glue) {GLUE\\metrics};
\node[metric, right=of head_chexpert] (metric_chexpert) {CheXpert\\metrics};
\node[metric, right=of head_mustard] (metric_mustard) {Ext. MUStARD\\metrics};

\draw[->] (glue) -- (enc_glue);
\draw[->] (chexpert) -- (enc_chexpert);
\draw[->] (mustard) -- (enc_mustard);

\draw[->] (enc_glue) -- (head_glue);
\draw[->] (enc_chexpert) -- (head_chexpert);
\draw[->] (enc_mustard) -- (head_mustard);

\draw[->] (head_glue) -- (metric_glue);
\draw[->] (head_chexpert) -- (metric_chexpert);
\draw[->] (head_mustard) -- (metric_mustard);

\end{tikzpicture}
\caption{High-level overview of the experimental pipeline: Dataset $\rightarrow$ Backbone encoder $\rightarrow$ MTL head variant $\rightarrow$ Evaluation metrics}
\label{fig:exp_pipeline_small}
\end{figure}
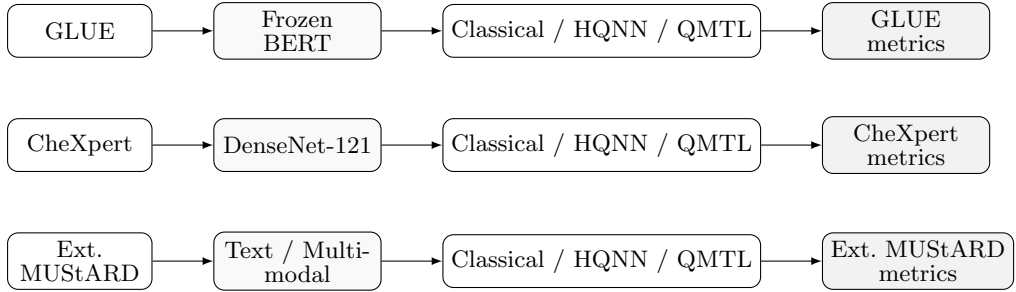

\subsubsection*{Implementation details}
\label{subsec:training_protocol}
All experiments were implemented in Python 3.11.13 using PyTorch 2.7.1 for classical components and PennyLane (v0.42.3) for the quantum components, and were executed with GPU acceleration on an NVIDIA A100 system. 
Because the datasets differ in scale, task structure, and supervision patterns, we tune hyperparameters for each dataset and model variant individually using grid search and report the final configurations in Appendix~\ref{appendix:methods}. 
Table~\ref{tab:exp_summary} provides a compact overview of the model variants considered per dataset, including the backbone configuration, head families, and the training protocol used.
\begin{table}[h]
\centering
\footnotesize
\caption{Summary of model variants, backbone configuration, and training procedures used across datasets. Detailed hyperparameters are provided in Appendix~\ref{appendix:methods}. 
}
\label{tab:exp_summary}
\renewcommand{\arraystretch}{1.5}
\begin{tabular}{m{1.8cm} m{2.9cm} m{3.2cm} m{3.4 cm}}
\toprule
Dataset &
Head variant &
Head detail &
Training method \\
\midrule
\addlinespace[0.4em]
\multirow{3}{*}{\makecell[l]{\\\textbf{GLUE}\\(Frozen BERT\\backbone)}}
& Classical MTL head 
& 1 linear layer per task
& \multirow{3}{*}{\makecell[l]{\\Task-sampled MTL (with\\ capped batches and \\LR scheduler)\\ Classification:\\ \code{BCEWithLogitsLoss},\\
\code{CrossEntropyLoss}, \\
Regression: \code{MSELoss}}} \\
\cmidrule{2-3}
& HQNN (4 or 10 qubits)
& VQC \citep{phukan_hybrid_2024} followed by 1 linear layer per task
& \\
\cmidrule{2-3}
& QMTL (10 qubits)
& QMTL head ($S_{1\to8} = 1$ \& $S_9=2$)
& \\
\midrule
\addlinespace[0.4em]
\multirow{3}{*}{\makecell[l]{\\\textbf{CheXpert} \\(DenseNet121 \\backbone)}}
& Classical MTL head
& 1 linear layer per task
& \multirow{3}{*}{\makecell[l]{\\Parallel tasks MTL\\Masked loss for missing \\labels;\\
Class-weighted focal loss\\
(Appendix Eq.~\ref{eq:focal_loss_app}) \\
3-class train, binary \\evaluations}}\\
\cmidrule{2-3}
& HQNN (4 or 10 qubits)
& VQC \citep{phukan_hybrid_2024} followed by 1 linear layer per task
& \\
\cmidrule{2-3}
& QMTL (10 qubits)
& QMTL head ($S_t = 2$ for all tasks)
& \\
\midrule
\addlinespace[0.4em]
\multirow{3}{*}{\makecell[l]{\textbf{Extended}\\\textbf{MUStARD}\\(Multimodal\\encoder \\backbone)}}
& Classical MTL head
& 1 linear layer per task
& \multirow{3}{*}{\makecell[l]{\\Parallel tasks MTL with \\weighted loss\\Binary: \code{BCEWithLogitsLoss} \&\\ Multiclass: \code{CrossEntropyLoss}}} \\
\cmidrule{2-3}
& HQNN (4 or 13 qubits) 
& VQC \citep{phukan_hybrid_2024} followed by 1 linear layer per task
& \\
\cmidrule{2-3}
& QMTL (13 qubits)
& QMTL head ($S_1 = 1$, $S_{2,3}=2,S_{4,5}=4$)
& \\
\bottomrule
\end{tabular}
\end{table}

\subsubsection*{Training protocols across labeling schemes }
We follow benchmark-standard supervised training protocols dictated by the label distribution and supervision structure of each dataset (Table~\ref{tab:dataset_table}). In particular, the datasets differ in how task labels are provided for each sample: single-task labeling, partially observed multi-task labels, or fully labeled samples. This determines whether a task-sampled, masked, or parallel multi-task training protocol is appropriate.

For \textbf{GLUE}, each sample provides a label for exactly one task, so we adopt task-sampled multi-task training where only the corresponding task head (and its associated QMTL components) is active per update. For this dataset, we repeat the experiments 5 times with 5 randomly generated seeds kept identical across the classical, both HQNNs, and our proposed quantum variants, and report the mean and standard deviation for each metric in bar chart format. This follows standard GLUE evaluation practice due to the high label imbalance in the data, keeps computation manageable for the quantum model, and ensures that any performance differences among the distinct model variants are not artifacts of a particular random initialization or task sampling. 

For \textbf{CheXpert}, labels are partially observed across tasks; we therefore use a masked loss so that only labeled tasks contribute to the objective for each sample.

For \textbf{Extended MUStARD}, each sample is fully labeled across tasks, enabling standard parallel multi-task optimization with a weighted sum of task losses. 

The experiments on CheXpert and Extended MUStARD are repeated for five non-overlapping folds of the dataset, and we report the mean and standard deviation (error bars) in bar chart format. We present our results and discuss our findings when comparing the multi-task performance of our proposed QMTL architecture against baselines across the three benchmark datasets. 

\subsection{Simulation Results}
In this section, we present our findings by firstly showing that our proposed architecture has significantly fewer parameters than the baseline considered in this work. Then, we show that the performance of QMTL is comparable to and, in some cases, even better than that of baselines while being more parameter-efficient.

\vspace{1em}
\noindent
\textbf{Parameter efficiency across datasets}\quad

\noindent
A key motivation for our proposed QMTL head module is to leverage the superposition property and the ability to encode in a high-dimensional latent space using a compact quantum circuit, i.e., one with minimal parameter count and quantum resources. 
To evaluate this, we compared the total number of trainable parameters for both our proposed model and competing baselines for multi-task learning on each of the datasets that we studied in this work.
\begin{figure}[h!]
    \centering
    \includegraphics[width=0.8\linewidth]{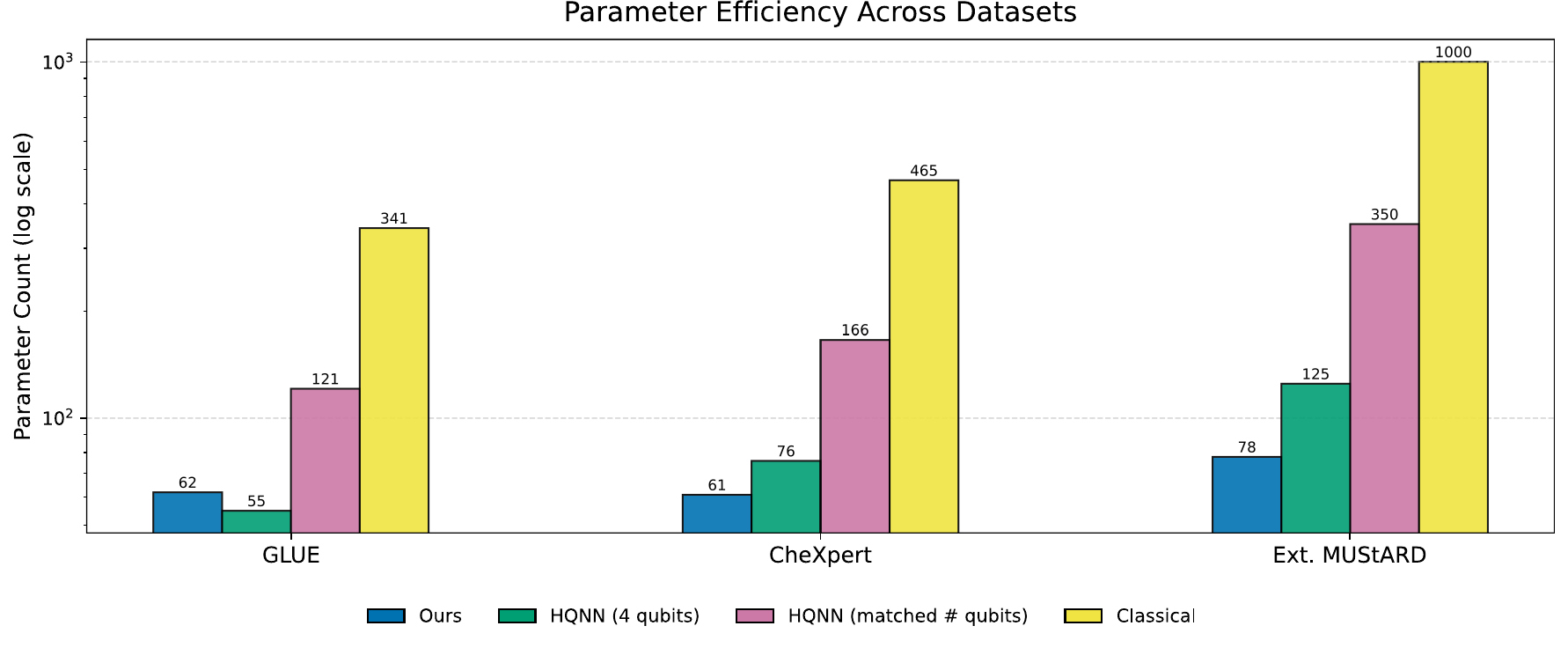}
    \caption{Total trainable parameters in task-specific heads across datasets. QMTL can achieve more than $12\times$ reduction compared to classical and 4-5$\times$ compared to HQNN (when scaled to equivalent qubit count)}
    \label{fig:param_efficiency}
\end{figure}
Fig.~\ref{fig:param_efficiency} shows that our QMTL architecture uses substantially fewer parameters across all datasets. The classical heads require $\mathcal{O}(T^2)$ parameters per task, accumulating to hundreds; QMTL's shared encoder and lightweight task-specific subcircuits require only $\mathcal{O}(T)$, reducing the parameter cost to tens. This supports the parameter scaling analysis stated in Theorem \ref{thm:param_scaling}. 

\subsubsection{Results on GLUE dataset}
\begin{figure}[h!]
    \centering
    \begin{subfigure}{\linewidth}
        \centering
        \includegraphics[width=\linewidth]{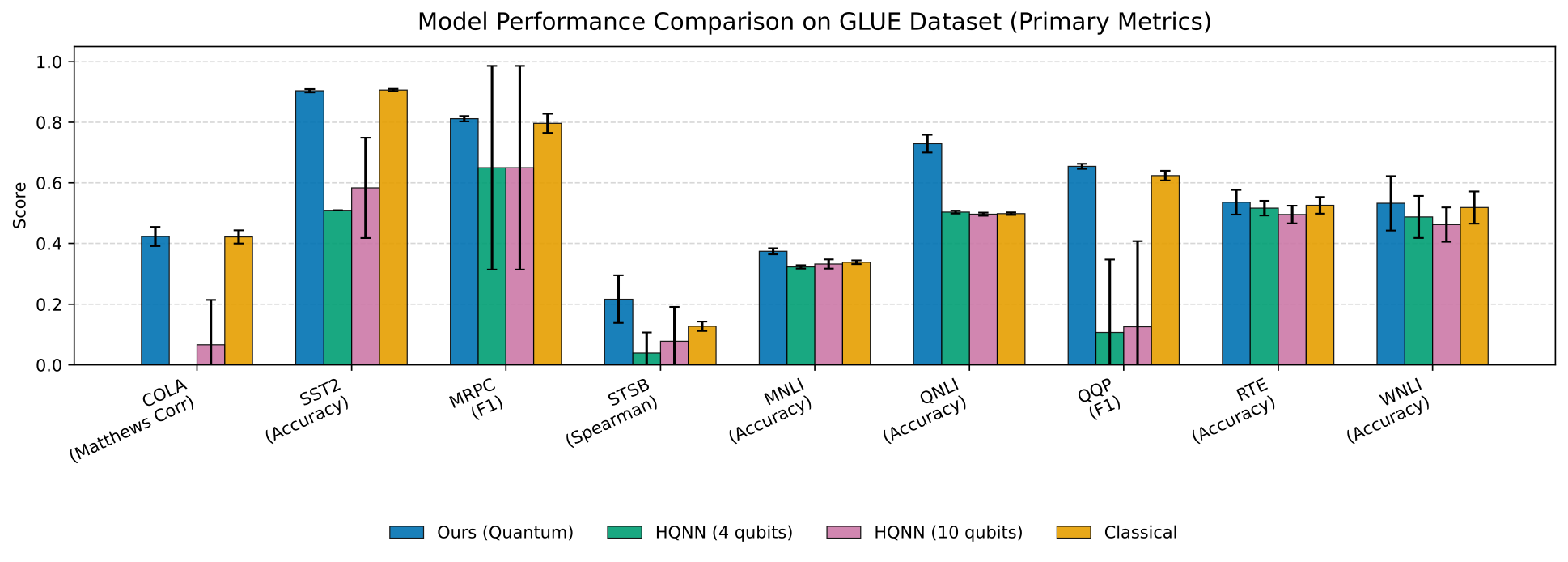}
        \caption{Performance with respect to primary metrics}
        \label{fig:glue_primary}
    \end{subfigure}
    \hfill
    \begin{subfigure}{\linewidth}
        \centering
        \includegraphics[width=\linewidth]{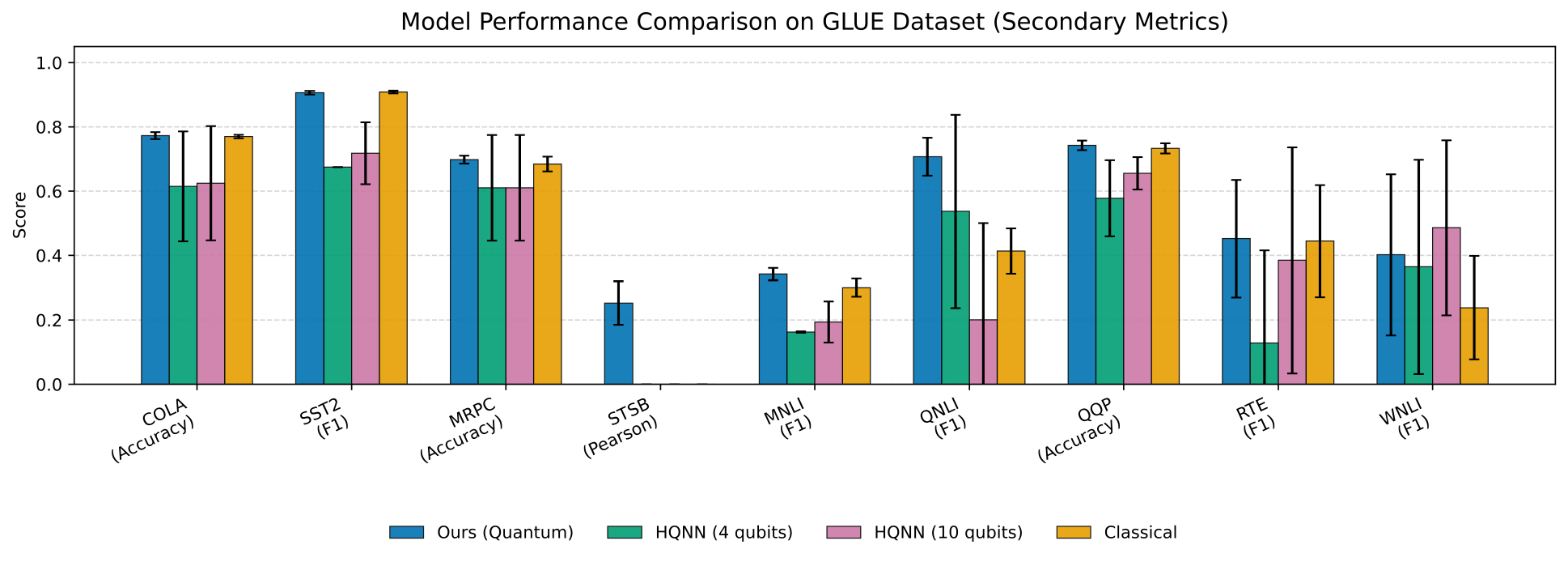}
        \caption{Performance with respect to secondary metrics}
        \label{fig:glue_secondary}
    \end{subfigure}
    \caption{Multi-task performance on the GLUE benchmark across four MTL head variants}
    \label{fig:glue_performance}
\end{figure}

On GLUE, QMTL achieves performance on par with the classical baseline across both primary and secondary metrics (Fig.~\ref{fig:glue_performance}). The quantum head matches or exceeds the classical head on several tasks (e.g., MRPC, QQP, QNLI) and remains within a small margin on others, while using substantially fewer trainable parameters. Both HQNN variants consistently underperform, likely due to the information bottleneck induced by the compressed projection into a small quantum block followed by classical post-processing layers. The regression task (STS-B) yields lower correlations across all variants, as expected given a shallow regression head and a limited label range; this limitation is not specific to QMTL but reflects a conservative, capacity-matched design for fair comparison.

\textit{\textbf{Remarks:} QMTL attains classical-level GLUE performance while outperforming HQNN baselines under comparable qubit budgets.}


\subsubsection{Results on CheXpert dataset}
\begin{figure}[h!]
    \centering
    \includegraphics[width=\linewidth]{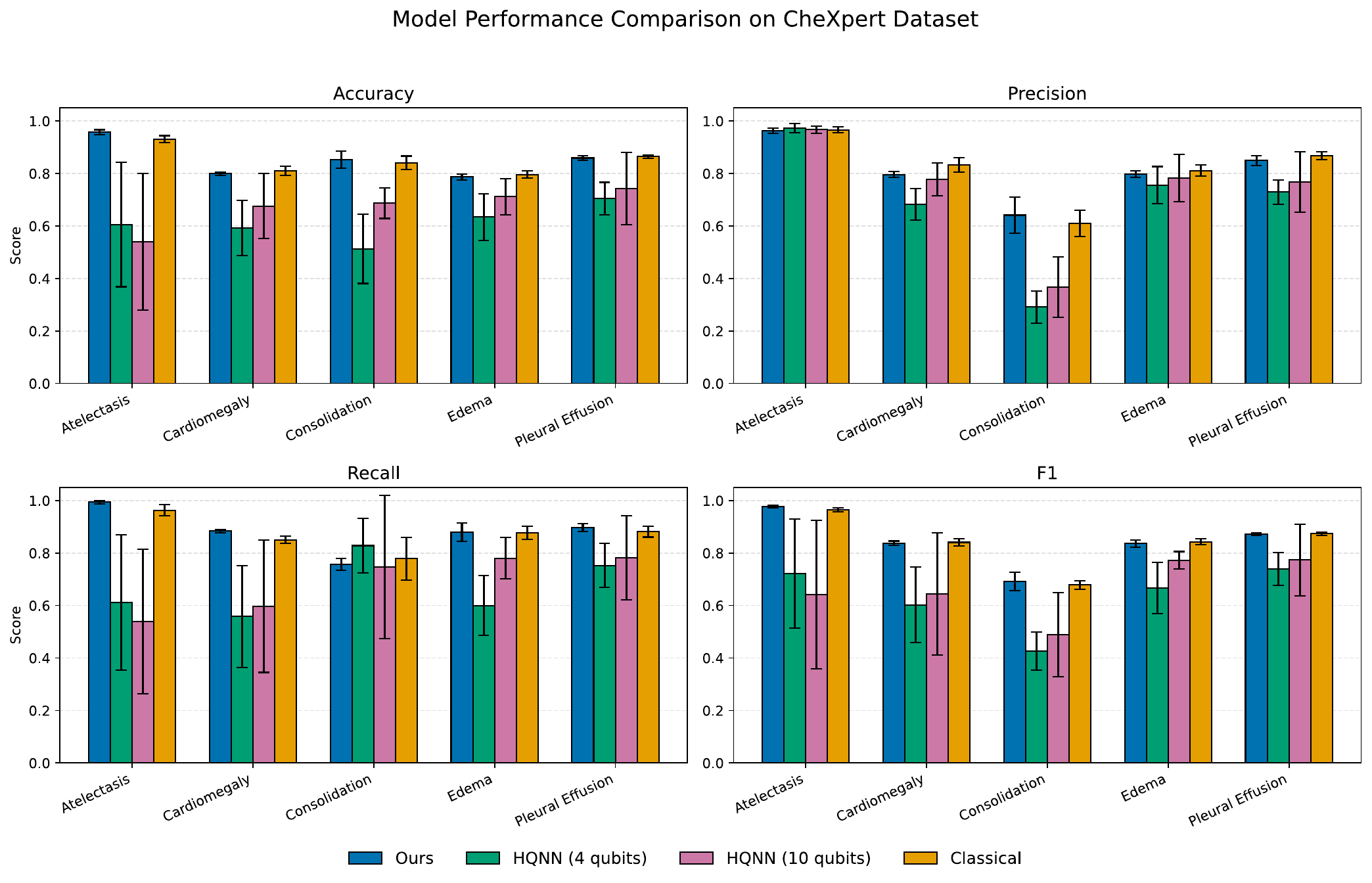}
    \caption{Multi-task performance on the CheXpert dataset across four MTL head variants}
    \label{fig:chexpert_performance}
\end{figure}

On CheXpert, QMTL is competitive with the classical head across the five pathology tasks and consistently outperforms both HQNN baselines (Fig.~\ref{fig:chexpert_performance}). QMTL achieves comparable Accuracy/F1 on several tasks (e.g., Atelectasis and Edema) and remains within fold-to-fold variability on the more challenging tasks (e.g., Consolidation). Increasing HQNN qubits (4$\rightarrow$10) does not yield a systematic improvement, suggesting that head design (task-local specialization with multi-observable readout) matters more than raw qubit count in this setting. Error bars across folds are comparable to the classical baseline, indicating that parameter efficiency does not come at the expense of stability.

\textit{\textbf{Remarks:} QMTL matches classical performance on key CheXpert tasks while exceeding HQNN baselines, and scaling HQNN qubits alone does not close the gap.}


\subsubsection{Results on Extended MUStARD dataset}
\begin{figure}[htbp!]
    \centering
    \includegraphics[width=\linewidth]{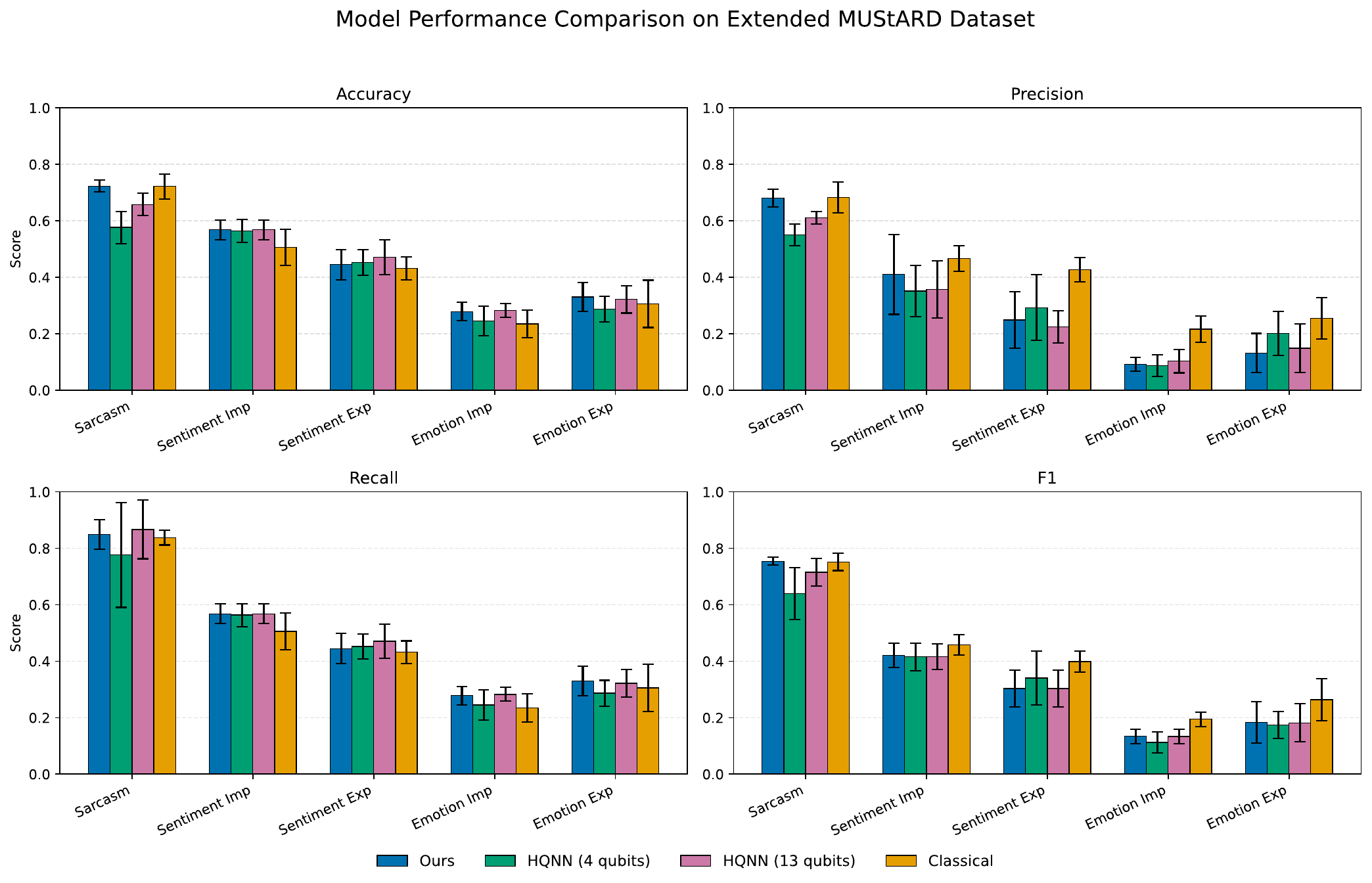}
    \caption{Multi-task performance on the Extended MUStARD dataset across four MTL head variants}
    \label{fig:mustard_performance}
\end{figure}

On Extended MUStARD, QMTL is consistently competitive and often strongest across the five tasks and four metrics (Fig.~\ref{fig:mustard_performance}). QMTL performs particularly well on sarcasm (high Accuracy/F1 and strong Recall), while remaining comparable to the classical head on sentiment and the harder 9-class emotion tasks. Both HQNN variants lag behind despite using comparable or more quantum resources; moreover, increasing the HQNN qubit count (4$\rightarrow$13) provides no clear benefit, reinforcing that the bottleneck lies in the HQNN method's projection and classical post-processing pipeline rather than the qubit count. Performance variability across folds remains moderate and comparable across methods.

\textit{\textbf{Remarks:} QMTL provides robust multi-task performance on multimodal MUStARD and consistently outperforms HQNN baselines even when HQNN uses more qubits.}


\subsubsection{Cross-benchmark summary across datasets}
Across GLUE, CheXpert, and extended MUStARD, the proposed QMTL head achieves performance comparable to the strongest classical multi-task heads while using substantially fewer trainable parameters. In all three domains, QMTL consistently outperforms the HQNN baselines under the same experimental protocol, and increasing the qubit count in HQNN (e.g., 4$\rightarrow$10/13) does not produce systematic gains. 

Notably, performance remains stable across heterogeneous tasks, with no evidence of systematic degradation in any single task, suggesting that the shared-encoding and task-local specialization design mitigates task interference that commonly arises in multi-task settings. 

Taken together, these empirical findings support the parameter efficiency analysis in Section~\ref{subsec:param_efficiency}: the proposed quantum head achieves linear-in-$T$ parameter scaling while remaining competitive with classical heads whose parameter count grows quadratically with the number of tasks.


\section{Ablation studies and noisy simulations}\label{sec:ab_studies}
In this section, we conduct further experiments to support the parameter-efficiency claim of our model and to show the importance of entanglement on our circuit design. Then, to better understand the performance limits under noisy conditions, we deploy our proposed QMTL architecture on noisy simulators with increasing noise levels.
\subsection{Effect of parameter budget on model performance}
\label{subsec:ab_param}
We conduct an ablation study to analyze the effect of increasing the number of parameters in our QMTL architecture as described in Section \ref{sec:our_model}. We do this by varying the quantum encoding depth $L$ and the task-specific head depth $L_h$. Fig.~\ref{fig: param_sweep} reports mean Accuracy and F1-score across the five CheXpert tasks, with shaded regions indicating the standard deviation over five folds. 

\begin{figure}[ht]
    \centering
    \includegraphics[width=\linewidth]{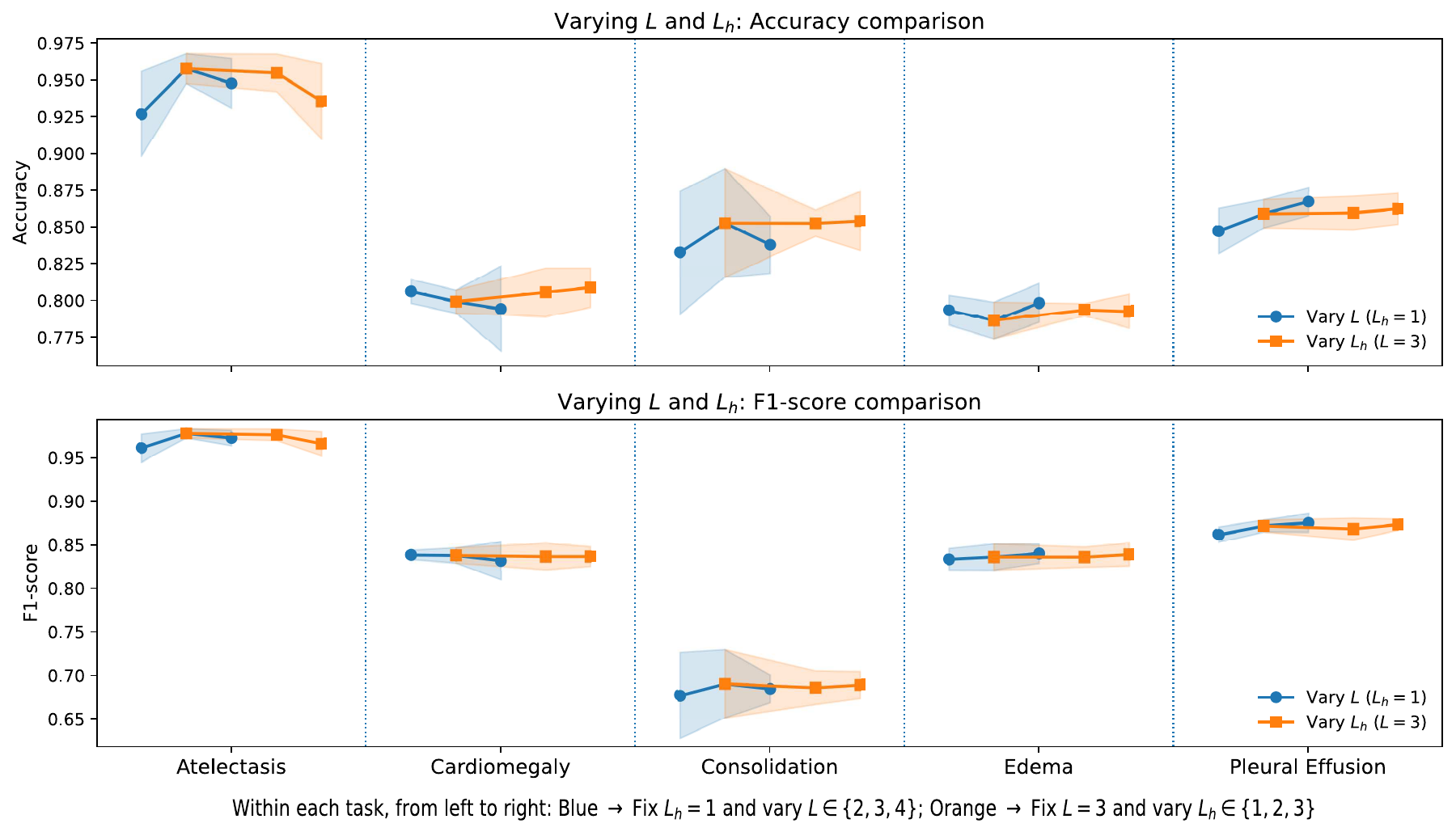}
    \caption{
    Effect of increasing parameter cost on model performance. The \textcolor{blue}{blue} linegraph represents the trend observed when varying the encoding layer depth $L \in \{2,3,4\}$, while keeping the task-specific subcircuit depth $L_h=1$. The \textcolor{orange}{orange} linegraph shows the trend when we vary the task-specific subcircuit depth $L_h \in \{1,2,3\}$ and keep the encoding layer depth $L=3$. A moderately deep encoding subcircuit coupled with shallow task-specific subcircuits provides the best performance, highlighting the best balance between expressivity of the shared quantum representation and parameter efficiency.}
    \label{fig: param_sweep}
\end{figure}

When varying the encoding depth $L\in \{2,3,4\}$ while fixing $L_h=1$, we observe a consistent improvement from $L=2$ to $L=3$ across all tasks and metrics, indicating that deeper shared encodings enable quantum representations that generalize better across tasks, reducing the need for deep task-specific heads. Increasing the depth further to $L=4$ yields marginal or no additional gains, suggesting saturation of representational capacity and highlighting the importance of balancing expressivity with parameter efficiency in the Noisy Intermediate-Scale Quantum (NISQ) regime \citep{Preskill_2018}. 

In contrast, when fixing $L=3$ and increasing the task-specific depth $L_h$, performance remains largely unchanged across tasks. This indicates that a single task-specific quantum layer is sufficient to adapt the shared representation to individual objectives in our design. Also, deeper task heads introduce additional parameters without improving generalization. Hence, this ablation study indicates a diminishing-returns effect: once the quantum head is moderately expressive, further increases in the number of parameters do not yield substantial performance gains. 

Based on these findings, we set the hyperparameters $L=3$ and $L_h=1$ in our experiments, as this configuration provided the best trade-off between performance, stability, and parameter efficiency. These results empirically support the design principle of QMTL, where most of the learning capacity is allocated to the shared quantum encoding, while task-specific components remain lightweight.

\textit{\textbf{Remarks:} This study demonstrates that performance gains in QMTL primarily stem from increasing shared quantum encoding depth rather than task-specific capacity, validating our choice of a moderately deep shared encoder ($L=3$) paired with minimal task-specific heads ($L_h=1$).}

\subsection{Effect of entanglement on performance}
\label{subsec:ab_ent}

To isolate the effect of entanglement in the shared encoder, we compare two variants of the proposed quantum head while keeping the backbone network, training protocol, and task-specific subcircuits fixed. In both settings, we use the same number of qubits $Q$, encoding depth $L$, and number of trainable single-qubit rotations per layer $k_\theta$ in the task-specific subcircuits (Section~\ref{subsec:our_model}). We vary only whether the shared encoding stage includes entangling gates, so any performance differences can be attributed to the presence of entangling gates.
In our proposed QMTL architecture (Fig.~\ref{fig:our_qmtl_circuit}), entanglement between qubits is introduced through fixed CNOT gates in the task-independent subcircuit (shared encoder). These two-qubit gates are not trainable and therefore do not increase the number of learnable parameters; instead, they control how information is distributed across qubits. To assess the role of entanglement, we compare the full model against a non-entangled variant in which all CNOT gates are removed from the shared encoder. In this setting, the encoder consists solely of local, parameterized single-qubit rotations that prepare a product quantum state. While the amplitudes of this state still depend nonlinearly on the latent features $\Vec{Z}$, no multi-qubit correlations are introduced.

We evaluate both variants on the CheXpert dataset and report Accuracy and F1-score using five-fold cross-validation. Results are summarised in Fig.~\ref{fig:entanglement}.

\begin{figure}[ht]
    \centering
    \includegraphics[width=\linewidth]{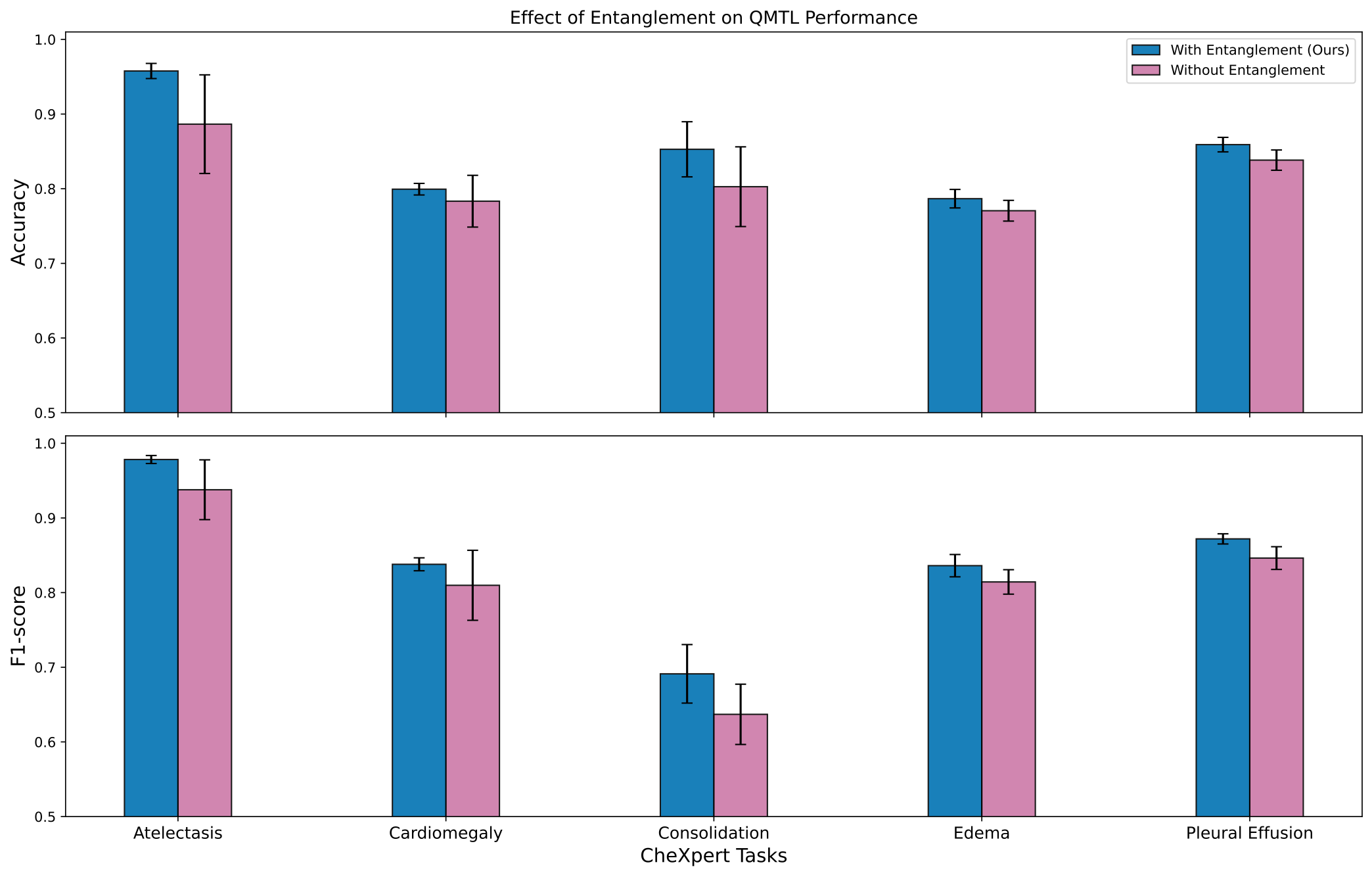}
    \caption{Effect of entanglement in the shared encoder on QMTL performance across CheXpert tasks. These results suggest that entanglement in the shared encoder can improve the multi-task performance of QMTL.}
    \label{fig:entanglement}
\end{figure}

Fig.~\ref{fig:entanglement} compares QMTL with an entangled shared encoder against an otherwise identical variant in which all entangling CNOT gates are removed from the encoder. Across both Accuracy and F1-score, incorporating entanglement yields a consistent improvement: the entangled variant achieves higher mean performance for every task, indicating that entangling operations in the shared stage systematically strengthen downstream multi-task predictions.

Across folds, the entangled encoder also exhibits comparable or, in several cases, slightly reduced variability relative to the non-entangled encoder, suggesting that the shared entangling structure can improve reliability in addition to average performance. Importantly, these gains do not arise from increased trainable capacity: the entangling gates are fixed (non-trainable), and the task-specific subcircuits are unchanged, so both variants have the same number of learnable parameters. Overall, the results support the role of encoder entanglement as a practical mechanism for enriching the shared representation in a multi-task setting, improving performance without sacrificing parameter efficiency.

\textit{\textbf{Remarks:} Adding fixed entangling gates in the shared encoder consistently improves QMTL performance on CheXpert without increasing the trainable parameter count, reinforcing entanglement as a useful contributor to parameter-efficient shared representations.}

\subsection{Noise Sensitivity under Depolarizing Models}
\label{subsection:noise_sim}

To study robustness to hardware-like noise in a controlled setting, we evaluate our proposed QMTL architecture under depolarizing noise using Qiskit's \code{AerSimulator}. Independent depolarizing channels are applied to both single-qubit and two-qubit gates with error probabilities $p_1$ and $p_2$, respectively. For simplicity, we set $p_1 = p_2 = p$ and vary $p \in \{0.01, 0.05, 0.1, 0.2\}$.

Inference is performed using the same fixed trained checkpoint obtained from the first fold of the CheXpert experiment in Section~\ref{sec:exp_feasibility}, thereby isolating the effect of execution noise from training variability.
\begin{figure}[h!]
    \centering
    \includegraphics[width=\linewidth]{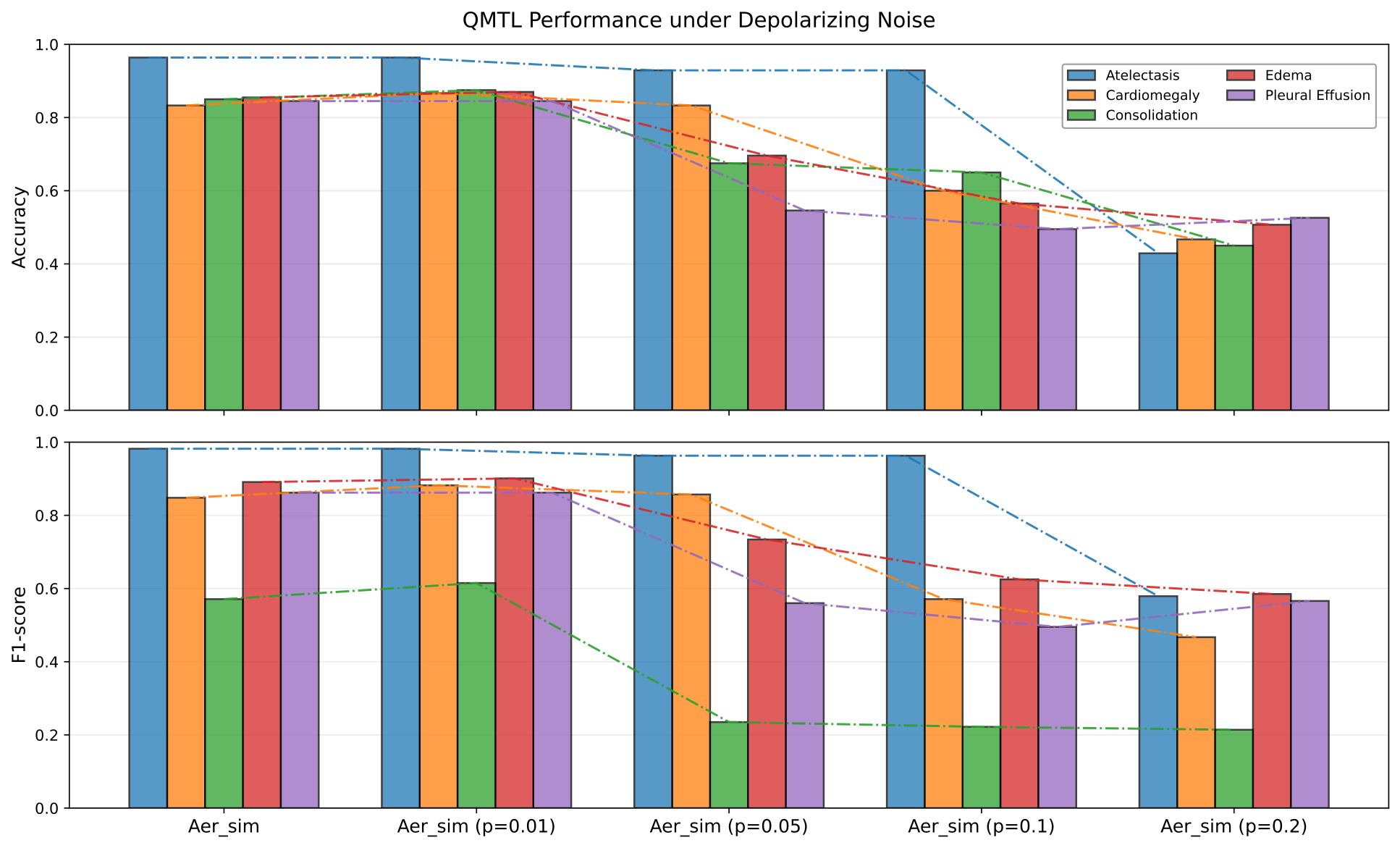}
    \caption{
    Multi-task performance of QMTL on the CheXpert dataset under increasing depolarizing noise levels in \code{AerSimulator}. Accuracy (top) and F1-score (bottom) are reported across CheXpert tasks using identical trained weights.
    }
    \label{fig:noise_results}
\end{figure}
Fig.~\ref{fig:noise_results} shows that performance remains close to the ideal simulator at low noise levels ($p=0.01$), but degrades progressively as $p$ increases, with F1-score typically exhibiting greater sensitivity than accuracy. This behavior is consistent with the accumulation of gate errors and reduced state fidelity in deeper circuits.

\textit{\textbf{Remarks:} Overall, these results indicate that QMTL performance degrades as expected under increasing depolarizing noise and provide a controlled reference point for interpreting real-device behavior in the next section.}

\section{Hardware Deployment on IBM Quantum Devices}
\label{sec:real_device}
To evaluate practical feasibility, we deploy the trained QMTL model on two IBM Quantum backends: \code{ibm\_fez} and \code{ibm\_boston}. Inference is performed using a fixed trained checkpoint obtained from the first fold of the CheXpert experiment in Section~\ref{sec:exp_feasibility}. No retraining is conducted; only the execution backend is varied. All circuit parameters and shot counts are kept identical to the simulator configuration to isolate the effect of hardware noise.

\begin{figure}[h!]
    \centering
    \includegraphics[width=0.7\linewidth]{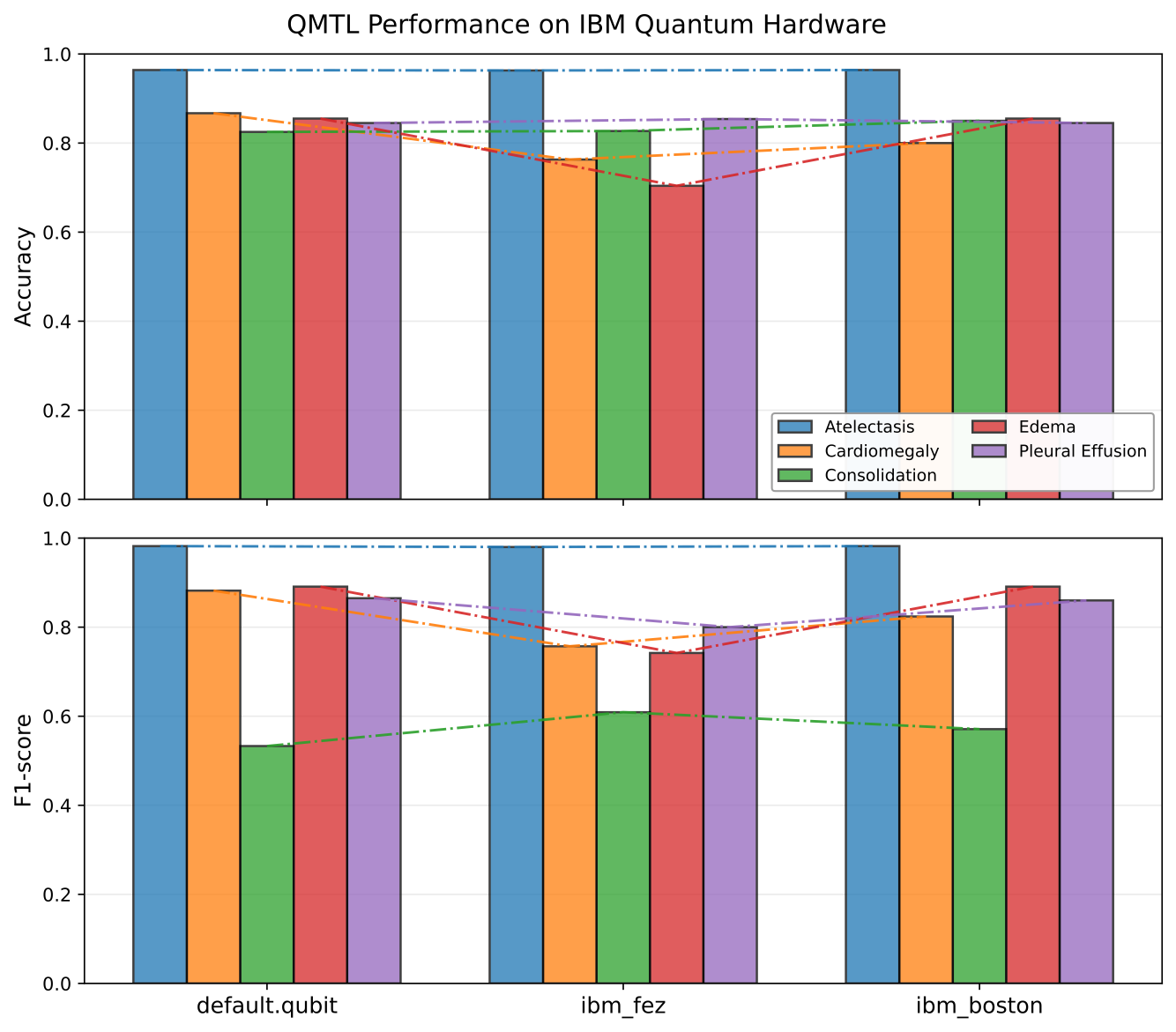}
    \caption{
    Multi-task performance of QMTL on IBM Quantum devices (\code{ibm\_fez}, \code{ibm\_boston}) compared to noise-free statevector simulation (\code{default.qubit}). Accuracy (top) and F1-score (bottom) are reported across the CheXpert tasks using the same trained weights. \code{ibm\_boston} performs closer to the ideal simulation than \code{ibm\_fez}
    }
    \label{fig:hardware_results}
\end{figure}

Fig.~\ref{fig:hardware_results} shows moderate degradation in performance on real hardware relative to ideal simulation, with variation across devices. In particular, \code{ibm\_boston} achieves performance closer to the ideal baseline than \code{ibm\_fez}. This is consistent with the reported median two-qubit error rates at the time of execution (approximately $1.15\times 10^{-3}$ for \code{ibm\_boston} and $2.55\times 10^{-3}$ for \code{ibm\_fez}). In several tasks, performance on \code{ibm\_boston} is comparable to the noisy simulator at $p=0.01$ and to the noise-free results, further supporting the consistency between controlled-noise modeling and observed hardware behavior.

\textit{\textbf{Remarks:} Despite hardware-induced degradation, QMTL retains meaningful multi-task performance across tasks, demonstrating that the proposed architecture can be executed on current devices with reasonable fidelity.}

\section{Discussion}\label{sec:discussion}
This section synthesizes the main findings of the paper and relates our parameter-efficiency analysis (Section~\ref{subsec:param_efficiency}) to the empirical results in Sections~\ref{sec:exp_feasibility}, \ref{sec:ab_studies}, and \ref{sec:real_device}.

\vspace{1em}
\noindent
\textbf{Performance of QMTL compared to baselines.}\quad
In Section~\ref{sec:exp_feasibility}, we evaluated QMTL on three multi-task benchmarks (GLUE, CheXpert, and extended MUStARD), spanning text, medical imaging, and multimodal inputs. Across these settings, QMTL is competitive with strong classical multi-head baselines and consistently outperforms HQNN-style baselines that compress high-dimensional features into a very small quantum latent space, followed by classical post-processing. These results suggest that a fully quantum multi-head module can support multi-task learning without relying on additional classical layers to map quantum outputs to logits, as in~\citep{phukan_hybrid_2024}. Notably, we do not observe systematic degradation of any single task relative to others, suggesting that localized task adaptation within a shared quantum representation can help limit task interference effects commonly associated with negative transfer in multi-task learning.

A recurring limitation of HQNN in our comparisons is the information bottleneck induced by explicit dimensionality reduction prior to the quantum circuit (e.g., projecting to a 3-dimensional input and then reading out only a small number of expectation values). While classical layers after the quantum block can reparameterize these outputs, information that is not represented in the quantum layer is difficult to recover. In contrast, QMTL avoids this bottleneck by (i) matching the circuit input dimension to a larger shared latent representation via qubit-efficient encoding, and (ii) using task-local subcircuits and multi-observable readout so that each task can extract an appropriately sized logit vector without introducing a large classical post-processing head. Empirically, this design yields stronger and more consistent performance than HQNN variants, even when HQNN qubit counts are increased to match QMTL. 

\vspace{1em}
\noindent
\textbf{Parameter scaling, efficiency, and empirical stability.}\quad
A primary motivation for QMTL is to reduce the number of trainable parameters in the multi-task head while retaining competitive accuracy. Under the classical multi-head baseline considered in Section~\ref{subsec:param_efficiency}, head parameter counts exhibit an overall quadratic dependence on the number of tasks, $\mathcal{O}(T^2)$. In contrast, QMTL uses qubit-efficient encoding and lightweight task-specific subcircuits, yielding a head-parameter scaling of $\mathcal{O}(T)$ with fixed circuit hyperparameters, while remaining competitive with the classical counterpart. Importantly, the empirical variability across folds/seeds remains comparable to the classical baseline, suggesting that parameter reduction does not come at the expense of stability in these experiments.

\vspace{1em}
\noindent
\textbf{Task difficulty and conservative head capacity.}\quad
The most challenging settings in our benchmarks (e.g., regression in STS-B and high-cardinality emotion classification in extended MUStARD) highlight that lightweight task-local heads, quantum or classical, can be capacity-limited. 
Because we intentionally constrained head capacity to enable controlled comparisons, the gaps observed on these tasks (Section~\ref{sec:exp_feasibility}) are most plausibly explained by task difficulty and label/data properties rather than a failure mode unique to QMTL.

\vspace{1em}
\noindent
\textbf{Ablation insights: depth and entanglement.}\quad
Sections~\ref{subsec:ab_param} and \ref{subsec:ab_ent} indicate that increasing circuit parameters (e.g., via larger encoding depth $L$ or task-specific depth $L_h$) does not monotonically improve performance; these quantities behave as hyperparameters and should be tuned for each setting. In addition, removing entangling CNOT gates from the shared encoder reduces performance, suggesting that encoder entanglement can improve the shared representation by introducing non-separable correlations between qubits.

\vspace{1em}
\noindent
\textbf{Hardware gap and noise sensitivity.}\quad
Beyond noise-free simulation, we assessed QMTL under depolarizing-noise simulation and on IBM Quantum hardware (Sections~\ref{subsection:noise_sim} and ~\ref{sec:real_device}). As expected, performance degrades relative to the ideal simulation and declines with increasing noise. Nevertheless, QMTL retains useful multi-task performance on real backends at the tested circuit sizes, and appropriately parameterized noisy simulation provides a practical proxy for hardware behavior in this regime.

\section{Conclusion}\label{sec:concl}

We introduce a parameter-efficient quantum multi-task learning (QMTL) head that combines a shared quantum encoding stage with lightweight task-local specialization and multi-observable readout. Across three multi-task benchmarks (GLUE, CheXpert, and extended MUStARD), QMTL achieves performance competitive with common baselines from prior work in multi-task learning. We achieve this with parameter scaling of only $\mathcal{O}(T)$, compared to $\mathcal{O}(T^2)$ in standard classical baselines. Further hardware and noisy simulator experiments demonstrate that QMTL can be successfully deployed on actual noisy devices and noisy simulators, while retaining meaningful multi-task performance on real backends at the circuit depth and qubit counts used in this study. 

Nevertheless, we identify some limitations in this work. First, the quantum heads are intentionally shallow and task-local to maintain a controlled parameter budget, which can limit expressivity for regression tasks and high-cardinality classification, particularly compared with deeper classical heads. Second, while we include evaluations under noisy simulation and on real backends for a subset of experiments, our results do not yet constitute a comprehensive hardware study across devices, noise models, and circuit optimizations. Third, our parameter-efficiency claims are made relative to the \emph{multi-task head} component (classical vs.\ HQNN vs.\ QMTL) rather than the full end-to-end pipeline: in all benchmarks, the classical backbone encoders (e.g., BERT, DenseNet121) dominate the total parameter count. Our focus is therefore on isolating the MTL head to assess whether task specialization can be achieved with markedly fewer trainable head parameters while remaining competitive. 

More broadly, our work suggests that parameter-efficient quantum circuit design could be explored as a complementary avenue to reduce the effective number of trainable (or active) parameters in large classical models, including large‑scale pretrained models, where head‑level parameter efficiency is a practical concern. This motivation stems from quantum feature maps' ability to exploit superposition and access high-dimensional Hilbert-space representations with compact parameterizations. Architectures similar to QMTL could also be investigated in a mixture-of-experts setting, where the goal is to scale model capacity while keeping the number of active parameters low for faster inference. Finally, robust deployment on near-term hardware will require extending QMTL with noise-aware objectives, measurement allocation strategies, and error-mitigation techniques; in this work, we only demonstrate hardware feasibility at inference time on a restricted subset of data.

\backmatter
\bmhead{Acknowledgements}
This work was supported by the University of Melbourne through the establishment of an IBM Quantum Network Hub at the University.

\section*{Declarations}
\begin{itemize}
\item Funding: No funding was received to assist with the preparation of this manuscript.
\end{itemize}








\begin{appendices}

\section{\label{appendix:datasets}Datasets, data preprocessing, and Evaluation Metrics}
We describe the three datasets used in this work to evaluate and compare the performance of our proposed quantum circuit as a multi-task head module against common classical and hybrid quantum–classical approaches. 

Table~\ref{tab:dataset_table} summarizes the datasets considered in this work, including modality, number of tasks, task type, dataset size, and labeling scheme. 
The labeling scheme is included because it determines the training protocol used for each benchmark (Section~\ref{subsec:training_protocol}). 

To evaluate performance across diverse tasks, we use commonly used metrics to ensure a fair and meaningful comparison with prior work. These metrics are as described in Appendix \ref{subsec:metrics}. 
\begin{table}[h]
\centering
\caption{Summary of datasets considered in this paper for comparing our proposed QMTL architecture against baselines.}
\setlength{\tabcolsep}{5pt}
\begin{tabularx}{\linewidth}{|Y|Y|c|Y|c|Y|}
\hline
Dataset & Modality & \# tasks & Task Type & \# samples & Labeling scheme \\
\hline
GLUE Benchmark & Text (NLP) & 9 & Binary, Multiclass, Regression & $\leq 300k$ per task & Exactly 1-labeled task per sample \\
\hline
CheXpert & Chest X-ray (image) & 5 & Binary & $20$k total & $\geq 1$ labeled task per sample \\
\hline
Extended MUStARD & Multimodal (audio, text, visual) & 5 & Binary, Multiclass & $690$ total & Fully labeled \\
\hline
\end{tabularx}
\label{tab:dataset_table}
\end{table}

\subsection{GLUE Benchmark Dataset}
In the domain of Natural Language Processing (NLP), an extensively used dataset for multi-task learning is the General Language Understanding Evaluation (GLUE) Benchmark Dataset. It consists of texts from various domains, including Wikipedia, social question-answer websites, and news articles. We consider this dataset in this work to benchmark the performance of our proposed model against commonly used multi-task classical and hybrid quantum neural network-based head structures. We describe these structures later in Appendix \ref{appendix:models}.  

The GLUE benchmark consists of nine tasks spanning single-sentence classification, similarity and paraphrase detection, and natural language inference. The single-sentence tasks include CoLA, which evaluates linguistic acceptability, and SST-2, a binary sentiment classification task. The similarity and paraphrase tasks include MRPC, QQP, and STS-B, which assess whether two sentences are semantically equivalent or, in the case of STS-B, predict a graded similarity score. The inference tasks comprise MNLI, QNLI, RTE, and WNLI, which require determining textual entailment or other logical relationships between sentence pairs. More details about the GLUE Benchmark dataset can be found in \citep{gluebenchmark}. We retrieve the datasets from the HuggingFace dataset repository (\href{https://huggingface.co/datasets/nyu-mll/glue}{nyu-mll/glue}), and adhere to the official GLUE task formats and train/validation splits for all tasks. For MNLI, which provides matched (in-domain) and mismatched (cross-domain) validation subsets, we concatenate these to form a single validation set.

For each dataset, preprocessing is performed through a unified loading pipeline which first extracts either a single text input (CoLA, SST-2) or a text pair following the standard GLUE schema: \code{premise}–\code{hypothesis} for MNLI task, \code{question}–\code{sentence} for QNLI, and paraphrase and similarity tasks are represented using \code{sentence1}–\code{sentence2} for MRPC, STS-B, RTE, WNLI, and \code{question1}–\code{question2} for QQP.
All tasks are tokenized using BERT's WordPiece tokenizer \code{bert-base-uncased} with a maximum sequence length of $T=128$, applying maximum length padding and truncation. Labels for classification tasks are kept as provided, while STS-B similarity scores are normalized to the range $[0,1]$ by dividing by 5.0, matching common GLUE regression practice and the internal expectations of our model heads.

After tokenization, all examples are converted into a uniform fixed-length representation (tokens) suitable for transformer-based models and grouped into training and validation splits. Since the test labels are not publicly available, all evaluations are carried out on the validation sets. For each task, we construct separate data loaders for training and validation, each with a batch size of 128. In the multi-task learning setting, the nine tasks are merged into a unified dataset, allowing the model to be trained across all GLUE tasks within a single optimization framework.

\subsection{CheXpert Dataset}
In clinical diagnosis, chest radiographs have been used to detect various pathologies. CheXpert \citep{CheXpertStanford2019} is a dataset containing 224,316 chest radiographs from 65,240 patients collected at Stanford Hospital \citep{irvin_chexpert_2019}. The authors proposed an automatic labeler to annotate each radiograph image with the presence, absence, or uncertain labels for 14 chest-related pathologies (e.g., Atelectasis, Edema, Pleural Effusion). In this work, we consider the downsampled version of this dataset (CheXpert-v1.0-small) available at \citep{CheXpertKaggle2022}, and we select 20,000 images from it: 10,000 from male patients and 10,000 from female patients.

As explained in \citep{irvin_chexpert_2019}, we consider the five \textit{competition tasks} selected by three board-certified radiologists based on clinical importance and prevalence in the validation set. Hence, in the subset we consider, each image is annotated for five thoracic pathologies (\textit{competition tasks}): Atelectasis, Cardiomegaly, Consolidation, Edema, and Pleural Effusion, with labels drawn from the original CheXpert schema: positive (1.0), negative (0.0), uncertain (-1.0), and blank (unlabeled). During preprocessing, all images were converted to grayscale, resized to 320 x 320 pixels, and normalized using ImageNet statistics to ensure compatibility with DenseNet-based encoders.
Label values were mapped to discrete class indices $\{0,1,2\}$ corresponding to uncertain, negative, and positive, respectively, while blank entries were assigned the value -100 and ignored during loss computation. We detail the training method for multiple models on this dataset in Section \ref{subsubsec:chexpert_train}.
This masking strategy enables training with partially labeled samples: for each image, only tasks with valid labels contribute to the loss, while unlabeled tasks are excluded.
A light data-augmentation pipeline comprising random horizontal flips $(p=0.3)$ was applied to improve generalization.

We frame the classification of multiple pathologies from a single chest radiograph as a multi-task learning problem, with each pathology treated as an independent prediction task. While prior works on Chest X-ray imaging, e.g.,  \citep{okolo_cln_2025}, extend this paradigm to include localization, we deliberately restrict our scope to classification. Bounding box annotations are scarce, available only for a small subset of the chest x-ray images (e.g., National Institute of Health’s (NIH) ChestX-ray8 dataset \citep{wang_chestx-ray8_2017}), and are absent altogether in CheXpert. Including localization would therefore have significantly reduced the available training data and compromised comparability with existing classification-focused baselines. Moreover, our focus is on evaluating parameter-efficient multi-task representation learning across diseases, rather than interpretability via spatial localization.

\subsection{Extended MUStARD Dataset}
The Multimodal Sarcasm Detection Dataset (MUStARD), compiled from popular television shows, was originally proposed by Castro et al. \citep{castro_towards_2019} for sarcasm detection. Later, Chauhan et al. released an extended version of this dataset that also includes sentiment and emotion (both implicit and explicit) classes \citep{chauhan_sentiment_2020}. The annotated extended MUStARD dataset comprises conversational audio-visual utterances. It contains 690 samples, each associated with textual, acoustic, and visual features, as well as contextual information from preceding dialogue turns. In this work, we used the pre-extracted feature representations provided by the authors of the extended MUStARD dataset, encompassing utterance-level embeddings for all three modalities. Readers are referred to the original paper~\citep{chauhan_sentiment_2020} for the exact preprocessing and feature extraction procedures. In the hybrid-quantum domain, Phukan et al. \citep{phukan_hybrid_2024} proposed a HQNN for multi-task classification using the extended MUStARD dataset, and we use their model architecture as one of the baselines in this work.

The extended MUStARD dataset includes five supervised tasks: \textit{sarcasm detection}, \textit{sentiment (implicit and explicit)}, and \textit{emotion (implicit and explicit)}. 
The \textit{sarcasm detection} task is a binary classification problem (sarcastic vs. non-sarcastic). 
\textit{Sentiment classification} is a 3-class task that involves predicting the explicit sentiment expressed in the utterance (positive, negative, or neutral) and the implicit sentiment inferred from contextual cues. 
The \textit{emotion recognition} tasks span nine emotion categories--disgust, excitement, sadness, neutral, frustration, happiness, fear, surprise, and anger--in both implicit and explicit settings. In this work, we aim to compare the performance of our proposed quantum MTL head module in predicting labels for each task against baseline models. 

\subsection{Evaluation Metrics}
\label{subsec:metrics}
Given the diversity of task formulations across datasets, we adopt task-specific evaluation metrics rather than enforcing a single metric across all experiments. 
This choice ensures comparability with prior work, as each benchmark traditionally reports performance using different quantitative measures. 
Table~\ref{tab:datasets_metrics} summarizes the tasks and metrics used for each dataset. 

\begin{table}[h]
\centering
\small
\caption{Summary of datasets, tasks, and evaluation metrics used across experiments. Symbols denote task types: $\dagger$ = Binary classification, $\ddagger$ = 3-class classification, $\clubsuit$ = 9-class classification, $\triangle$ = Regression.
Abbreviations: Acc = Accuracy, P = Precision, R = Recall, F1 = F1-score, Matthews\_corr = Matthews Correlation, Spearman = Spearman correlation, Pearson = Pearson correlation. For GLUE tasks, the primary metric is listed first.}
\label{tab:datasets_metrics}
\begin{tabular}{lll}
\toprule
\textbf{Dataset} & \textbf{Task} & \textbf{Metrics} \\
\midrule
\multirow{5}{*}{CheXpert}
 & Atelectasis$^{\dagger}$ & Acc, P, R, F1 \\
 & Cardiomegaly$^\dagger$ & Acc, P, R, F1 \\
 & Consolidation$^\dagger$ & Acc, P, R, F1 \\
 & Edema$^\dagger$ & Acc, P, R, F1 \\
 & Pleural Effusion$^\dagger$ & Acc, P, R, F1 \\
\midrule
\multirow{9}{*}{GLUE}
 & COLA$^\dagger$ & Matthews\_corr, Acc \\
 & SST-2$^\dagger$ & Acc, F1 \\
 & MRPC$^\dagger$ & F1, Acc \\
 & STS-B$^\triangle$ & Spearman, Pearson \\
 & MNLI$^\ddagger$ & Acc, F1 \\
 & QNLI$^\dagger$ & Acc, F1 \\
 & QQP$^\dagger$ & F1, Acc \\
 & RTE$^\dagger$ & Acc, F1 \\
 & WNLI$^\dagger$ & Acc, F1 \\
\midrule
\multirow{5}{*}{Extended MUStARD}
 & Sarcasm$^\dagger$ & Acc, P, R, F1 \\
 & Sentiment (Implicit)$^\ddagger$ & Acc, P, R, F1 \\
 & Sentiment (Explicit)$^\ddagger$ & Acc, P, R, F1 \\
 & Emotion (Implicit)$^\clubsuit$ & Acc, P, R, F1 \\
 & Emotion (Explicit)$^\clubsuit$ & Acc, P, R, F1 \\
\bottomrule
\end{tabular}
\end{table}
\section{\label{appendix:models}Baseline models from prior work}
In this section, we describe the three models obtained from multi-task learning literature from both the classical and hybrid quantum machine learning domains. We further explain how we adapt their backbones/feature extractors for different MTL head modules, providing a comprehensive method for benchmarking the performance of our proposed model across various datasets.

\subsection{\label{subsubsec: nlp_model}Baseline 1: Natural language processing (NLP) based multi-task learning on GLUE dataset}
We propose this model to perform MTL on the GLUE benchmark dataset. This model uses a hierarchical feature extractor composed of a pre-trained transformer followed by lightweight sequence modelling. Given an input token sequence of length $T = 128$, the transformer (\code{bert-base-uncased}) produces contextual token embeddings $\mathbf{H}=[\mathbf{h}_1,\dots,\mathbf{h}_T]\in\mathbb{R}^{T\times d_b}$, where $d_b=768$ corresponds to the hidden size of BERT-base. These embeddings are processed by a two-layer bidirectional LSTM with $\frac{T}{2}$ hidden units per direction (total hidden dimension \(d=768\)), producing time-indexed BiLSTM states \(\mathbf{u}_t\in\mathbb{R}^d\) for \(t=1,\dots,T\). 

To obtain a fixed-length sentence representation, we apply masked elementwise max-pooling over the sequence (padding positions are excluded via the attention mask \(m_t\in\{0,1\}\)):
\begin{equation}
\tilde{\mathbf{u}}_t \;=\; 
\begin{cases}
\mathbf{u}_t & \text{if } m_t=1,\\[6pt]
-\infty & \text{if } m_t=0,
\end{cases}
\qquad
\mathbf{s} \;=\; \max_{t=1,\dots,T} \tilde{\mathbf{u}}_t,
\label{eq:pool}
\end{equation}
where the max is taken elementwise and produces the pooled sentence vector \(\mathbf{s}\in\mathbb{R}^d\). For single-sentence tasks, $\mathbf{s}$ is projected to a lower-dimensional latent vector $\mathbf{z}\in\mathbb{R}^{D}$ via a two-layer MLP with \code{tanh} activation. For sentence-pair tasks, we compute pooled vectors \(\mathbf{s}_1,\mathbf{s}_2\) for each sentence and form an interaction-aware representation by concatenating the two vectors together with their elementwise absolute difference and Hadamard product:
\begin{equation}
\mathbf{p} \;=\; \big[\,\mathbf{s}_1 \;;\; \mathbf{s}_2 \;;\; |\mathbf{s}_1 - \mathbf{s}_2| \;;\; \mathbf{s}_1 \odot \mathbf{s}_2\,\big].
\label{eq:pair}
\end{equation}
The vector $\mathbf{p}\in\mathbb{R}^{4d}$ is then projected through a similar MLP to yield the task-dependent latent vector $\mathbf{z}\in\mathbb{R}^{D}$, where $D = 30$ as per our design. 

Due to computational constraints, we use a stable, widely adopted pretrained encoder (\code{bert-base-uncased}) with its weights frozen rather than using more recent, higher-capacity transformer models (e.g., \code{RoBERTa}, \code{DeBERTa}, \code{T5}). We also note that alternative contextual encoders used in prior work, such as ELMo \citep{peters-etal-2018-deep} from AllenNLP, as mentioned in \cite{gluebenchmark}, are not compatible with current versions of \code{PyTorch} and cannot be integrated into our training pipeline.

Next, the above model architecture supports multiple task-specific head variants, enabling a systematic comparison between classical and quantum-enhanced approaches to multi-task learning using an NLP-based feature extractor on the GLUE dataset. We describe the different head architectures below:
\begin{itemize}
    \item \textbf{Ours:} Our proposed quantum circuit consists of $Q=10$ qubits, $L=3$ encoding layers, and $T=9$ tasks. This design choice is deliberate because the input dimension of the quantum circuit has to match the output dimension of the feature extractor ($M=Q\cdot L=30$). We set the number of qubits for each task-specific sub-circuit based on the task-type or number of classes for that task: i.e., $S_{t=1, \cdots, 8} = 1$ for binary classification tasks (e.g., COLA, SST-2, MRPC, QNLI, QQP, RTE, and WNLI) and for the regression task (STS-B), while $S_{t=9}=2$ for the 3-class classification task (MNLI). We use a learnable affine transformation for STS-B, applying a scaling parameter $\gamma$ (initialized at $1.0$) and a bias parameter $\beta$ (initialized at 0.5) to the raw quantum output, $z$, such that the predicted regression output, $\hat{y} = \beta + \gamma z$. This transformation is required so that the model can better match the similarity scores, which are normalized to $[0,1]$.
    
    For each of the task-specific sub-circuits, we use only $L_h = 1$ layer of \code{StronglyEntanglingLayer}. $L_h$ is a hyperparameter we explored, and preliminary experiments showed that $L_h=1$ yields the best overall performance across all metrics. The outputs of each sub-circuit correspond to expectation values of carefully chosen observables: single-qubit Pauli-Z measurements for binary/regression tasks. For 3-class tasks, we employ individual Pauli-Z measurements on each qubit to capture local quantum states, and a tensor product Pauli-X measurement to capture non-local quantum correlations between qubits, i.e., we apply the following set of observables: $\{Z_0,\, Z_1,\, X_0 X_1\}$. 
    \item \textbf{Classical:} This consists of $T=9$ task-specific linear layers where their input dimensions are $D=30$, and output dimensions correspond to the number of classes and task type. Specifically, for the regression and binary classification tasks, the output dimension is set to 1, while for 3-class classification task (MNLI) we set the output dimension to 3. We note that we keep the input dimension of the classical head modules similar to our proposed model to enable fair comparison across architectures. 
    \item \textbf{HQNN:} This corresponds to the HQNN model proposed by \cite{phukan_hybrid_2024}, whereby the latent features are first processed by a 4-qubit quantum circuit, and then passed through task-specific linear layers for classification or regression. We first use a linear layer to bring the latent space dimension $D=30$ down to 3, which is the input dimension of the quantum circuit proposed by \citep{phukan_hybrid_2024}. We then use these as inputs to the 4-qubit quantum circuit, and Pauli-Z expectation measurement on each qubit gives a 4-dimensional quantum output. This is the input fed into $T=9$ task-specific linear layers with input dimensionality of 4 and output dimensions corresponding to the task type and number of classes, similar to the classical method explained above. We also increase the quantum circuit size to 10 qubits to match the number used by our proposed architecture, allowing us to compare performance at similar qubit counts.
\end{itemize}

\subsection{\label{subsubsec: cnn_model}Baseline 2: Convolutional neural network (CNN) based multi-task learning on CheXpert dataset}
This model was developed for MTL on the CheXpert dataset \citep{irvin_chexpert_2019}, and we adapt it to benchmark our proposed quantum method on the same dataset.
Specifically, the model consists of a DenseNet121 feature extractor (pretrained on ImageNet), at the end of which we apply \code{ReLU} activation function, \code{AdaptiveAvgPool2d} and dropout before feeding into 2 hidden layers with \code{ReLU} activation function and batch normalization to bring the dimension down to $M$ ($M$ varies as per the head structure we describe below). We consider this CNN architecture specifically because it produced the best performance when compared to other CNN architectures considered in \cite{irvin_chexpert_2019}. The structure of the feature extractor stays constant across experiments to clearly highlight the contribution of our proposed quantum circuit for multi-task learning. During training, we do not freeze the backbone, i.e., the weights of the feature extractor are fine-tuned for the CheXpert dataset. 

To best benchmark the performance of our proposed quantum circuit as a classification head  for multi-task learning on the CheXpert dataset with the CNN-based feature extractor, we consider the following head structures:
\begin{itemize}
    \item \textbf{Ours:} Our proposed quantum circuit uses $Q{=}10$ qubits and $L{=}3$ encoding layers. For CheXpert we have $T{=}5$ tasks and allocate $S_t{=}2$ qubits to each task-specific head ($t=1,\dots,5$). We set the latent dimension to $M{=}Q\!\cdot\!L{=}30$. Each head uses a single \texttt{StronglyEntanglingLayer}, and we evaluate $\{Z_0, Z_1, X_0X_1\}$, yielding three expectations that serve as the three logits. 
    Logits are obtained from Pauli expectation values, scaled by a learnable temperature parameter to align their magnitudes with standard logit scales, which stabilizes training. 
    \item \textbf{Classical:} This consists of five task-specific classification heads (one per pathology), each consisting of linear layers with three output nodes representing 3 logits (negative, uncertain, positive). We set $M=30$ input dimensions for each task-specific linear layer to match the input dimension of our proposed quantum circuit. This structure is developed based on the classical MTL  model studied in \cite{irvin_chexpert_2019}. 
    \item \textbf{HQNN:} First, for this method, we set $M=3$ output dimensions as per \citep{phukan_hybrid_2024}. Then, we append the 4-qubit quantum circuit proposed in \citep{phukan_hybrid_2024}, whose four-dimensional output is scaled by a trainable parameter. The quantum circuit is followed by five linear layers (one per task), each with four input nodes and three output nodes, corresponding to three logits. We also study the performance of this HQNN head module when the number of qubits is increased to $10$, matching the number of qubits in our proposed MTL quantum circuit. This allows us to standardize the qubit resource consumed by either model.     
\end{itemize}

\subsection{\label{subsubsec:HQNN}Baseline 3: Hybrid quantum neural network (HQNN) based multi-task learning on extended MUStARD dataset}
Phukan et al. proposed a HQNN for multimodal multi-task sarcasm, sentiment, and emotion analysis \citep{phukan_hybrid_2024}. This model is one of the latest attempts from the literature to integrate quantum circuits into multimodal multi-task learning for natural language understanding. Specifically, the HQNN was applied to the extended MUStARD dataset \citep{chauhan_sentiment_2020} for sarcasm, sentiment (implicit/explicit), and emotion (implicit/explicit) recognition. 

The model consists of a segment-wise inter-modal attention-based framework and is described in more detail in \citep{chauhan_sentiment_2020}. In this work, we utilize the extracted feature vector representations from \citep{chauhan_sentiment_2020}, similar to how Phukan et al. evaluated their HQNN method for multi-task learning on the extended MUStARD dataset \citep{phukan_hybrid_2024}. We first require a shallow dense network to reduce the feature vector dimensionality to $M$, matching the input dimension of each classification head.
We consider the HQNN a benchmark because it sets a direct precedent for combining quantum layers with multimodal, multi-task problems. Moreover, it provides a valuable benchmark for evaluating quantum enhancements in multi-task learning. By benchmarking against HQNN, we can directly contrast our proposed framework with this quantum baseline on the extended MUStARD dataset and highlight differences in performance, parameter efficiency, and resource usage. More details about the model and the structure of the quantum circuit can be found in \citep{phukan_hybrid_2024}.

Below are the different classification heads we consider to benchmark our method:
\begin{itemize}
    \item \textbf{Ours:} We first project the latent to $M{=}39$ to match a $Q{=}13$, $L{=}3$ quantum input ($Q\!\cdot\!L{=}39$). The five MUStARD tasks use task-local subcircuits with one \texttt{StronglyEntanglingLayer} each: $S_1{=}1$ qubit (sarcasm, binary), $S_{2,3}{=}2$ qubits (sentiment, 3-class), and $S_{4,5}{=}4$ qubits (emotion, 9-class). Logits are taken directly from Pauli expectation values. For the binary head, we use $\langle Z_0\rangle$. For the 3-class heads, we use $\{Z_0, Z_1, X_0X_1\}$ (with observables evaluated in compatible measurement groups). For the 9-class heads we use $\{Z_0, Z_1, Z_2, Z_3\}$ together with their pairwise products $\{Z_0Z_1, Z_1Z_2, Z_2Z_3, Z_3Z_0\}$ and $\{X_0X_1X_2X_3\}$, again grouped by commutativity. These nine expectations match the required logit dimensionality. By using expectation values as logits (without additional dense layers), head parameters scale with locality and shallow depth rather than with output dimension.

    \item \textbf{HQNN:} This corresponds to the HQNN model proposed by Phukan et al. to perform multimodal multi-task classification on the extended MUStARD dataset. Specifically, they employed a dense network to project these features into a $M=3$-dimensional latent space, and then passed these as three parameters ($\beta$, $\gamma$, and $\theta$) into a 4-qubit quantum circuit. The ansatz is repeated three times as per \citep{phukan_hybrid_2024}. The quantum layer produces 4 Pauli-Z expectation values, one from each qubit, which are then fed into classical, task-specific linear layer heads. Each of these has four input nodes, and the number of output nodes corresponds to the number of classes for each of these tasks. We also scale the number of qubits of this quantum circuit to 13 to match the number of qubits used in our proposed quantum circuit. By keeping the number of qubits the same, we can directly compare performances under the same resource usage, and better benchmark our method. Specific details on the quantum circuit structure can be found in \citep{phukan_hybrid_2024}.  
    
    \item \textbf{Classical:} First, we employ two linear layers with $M=39$ output dimensions and \code{ReLU} activation function on the first layer to project the extracted features to the same number of input dimensions as our proposed quantum MTL head module. Next, we use five task-specific linear layers, each with input dimension $M=39$, and output dimension equal to the number of classes for each task. We set the input dimension to $M=39$ to match the input dimension of our proposed quantum head circuit for fair performance and parameter-cost comparison. 
\end{itemize}

\section{\label{appendix:methods} Detailed training and evaluation methods}
In this section, we describe how we train and evaluate the different model variants on each dataset. 
\subsection{\label{subsubsec:glue_train}Training and evaluation method on GLUE dataset}
We train all the models described in Section \ref{subsubsec: nlp_model} using a multi-task optimization strategy in which a single GLUE task is sampled at every training update, rather than jointly consuming batches containing mixed-task examples. For each epoch, we first fix a capped number of training batches per task (this is set to $300$ for large datasets, e.g., MNLI and QQP, which have above $3,000$ batches in the training set (batch size = $128$)). This setup prevents these very large datasets from dominating the optimization and provides a more balanced contribution from tasks with fewer samples (e.g., WNLI, RTE, MRPC, and CoLA). At each update, a task $t$ is drawn with probability proportional to the number of its remaining capped batches, a batch from that task is fetched, and converted to the unified input format (single-sentence or sentence-pair). This is then passed through the shared Bert-based encoder and the corresponding task-specific head. We compute the appropriate task loss (\code{MSELoss} for STS-B, \code{BCEWithLogitsLoss} for GLUE's binary classification tasks, and \code{CrossEntropyLoss} for multi-class tasks). The loss is then backpropagated with gradient clipped at a maximum norm of $1.0$, and all the trainable parameters are updated using the \code{AdamW} optimizer with learning rate $5\times 10^{-4}$ and weight decay $0.01$.

We define two sets of metrics: primary and secondary. These are largely adapted from the original GLUE benchmark, but here we explicitly distinguish between them: primary metrics follow the GLUE scoring protocol and drive model selection, while secondary metrics provide complementary diagnostic information. 

Training proceeds for up to $15$ epochs, with periodic validation every $2,000$ updates. At each validation checkpoint, we evaluate the model on the validation split of every GLUE task using the primary metric for that task (Matthews correlation for CoLA, accuracy for most classification tasks, and Spearman correlation for STS-B) and compute a macro-average across tasks. This macro-average drives a \code{ReduceLROnPlateau} scheduler (decay factor $0.2$, and minimum learning rate $10^{-7}$) and is also used for early stopping. We retain the checkpoint with the best macro-average and terminate training if no improvement is observed, after a fixed patience window of $5$ validation checks. We also highlight that this training protocol is largely adapted from the original GLUE benchmark \citep{gluebenchmark}. We also compute a secondary metric for each GLUE task. A summary of the primary and secondary metrics is provided in Table \ref{tab:datasets_metrics}. 

We note that in all experiments, we freeze the parameters of the BERT-based feature extractor and train only the task-specific head layers (quantum, HQNN, or classical). This design choice is deliberate: our goal is to isolate the contribution of the head architecture itself, enabling a fair comparison between our quantum, classical, and hybrid multi-task learning head architectures without confounding effects from large-scale encoder fine-tuning. Consequently, we do not aim to reach state-of-the-art GLUE performance as reported on the \href{https://gluebenchmark.com/leaderboard}{public leaderboard}. Our setup, therefore, focuses on comparing head architectures using a controlled, consistent representation-learning backbone.
The same training protocol, including task sampling, batch caps, optimizer, and scheduler, is applied identically to all the head variants to ensure a fair comparison.

\subsection{\label{subsubsec:chexpert_train}Training and evaluation method on CheXpert dataset}
Each model was trained in a multi-task three-class formulation, where each task-specific head predicts one of three categories--uncertain, negative, or positive. We fine-tune a DenseNet-121 backbone initialized with ImageNet weights (i.e., base weights are not frozen) and project backbone features to a shared low-dimensional representation for the diverse multi-task classification heads as described in Section \ref{subsubsec: cnn_model}. 

We highlight that the CheXpert dataset is highly task- and label-imbalanced. Hence, we adopt the following methods to train our models. First, to mitigate class imbalance, we compute inverse-frequency per-class weights for each task:
\begin{equation}
w_{c,t}=\frac{N}{K\cdot n_{c,t}},
\label{eq:sample_weight}
\end{equation}
where \(N\) is the total number of samples in the training split, \(K=3\) is the number of classes, and \(n_{c,t}\) is the count of class \(c\) for task \(t\). These per-class weights are applied as loss weights (see below). Next, training uses Focal Loss, $\mathcal{L}_{\mathrm{FL}}$, to emphasise hard examples. We use the \code{CrossEntropyLoss} computation to extract \(\log\hat{p}\) with \(\hat p\) the model probability for the true class. The focal loss is computed as below:
\begin{equation}
\mathcal{L}_{\mathrm{FL}}(\hat p) = -\alpha (1-\hat p)^\gamma \log\hat p,
\label{eq:focal_loss_app}
\end{equation}
with \(\gamma=2\) and \(\alpha=1\) unless stated otherwise. The per-task class weights \(w_{c,t}\) are incorporated into the loss (i.e., per-class weighting inside the cross-entropy computation). Missing labels are masked (set to a sentinel, e.g., \(-100\)) so they do not contribute to the loss or gradients. We apply gradient clipping (maximum norm \(=1.0\)) and optimise with \code{Adam} (learning rate \(10^{-4}\)), batch size \(=32\) for our proposed and classical model, and batch size \(=128\) for the HQNN models. We train the models for up to $30$ epochs with early stopping (patience \(=5\)) selected on the validation mean F1 across tasks, computed every two epochs, and backpropagate through the quantum circuits wherever relevant via the parameter-shift algorithm \citep{schuld_evaluating_2019}. We highlight that we adopt a training method similar to that proposed in \citep{irvin_chexpert_2019}.

For evaluation, 3-class outputs are converted to clinically interpretable binary probabilities by excluding uncertain labels and renormalizing the negative/positive probabilities:
\begin{equation}
\label{eq:bin_out}
p_{+}^{(\mathrm{binary})}=\frac{p_{+}}{p_{+}+p_{-}},
\end{equation}
where \(p_{-}\) and \(p_{+}\) are the model’s predicted probabilities for the negative and positive classes, respectively.
Binary predictions use a threshold of \(0.5\). We report Accuracy, Precision, Recall and F1 for the positive (clinical) class averaged across the five tasks. This approach allows the model to learn from uncertain examples during training while restricting evaluation to the clinically relevant positive/negative decision, consistent with Irvin \emph{et al.} \citep{irvin_chexpert_2019}.

\subsection{Training and evaluation method on extended MUStARD dataset}
We constructed each model specified in Section \ref{subsubsec:HQNN}, with the input being the feature vector representations of the multimodal extended MUStARD dataset and the output being 5 task-specific labels. In this training protocol, we optimize a weighted sum of task losses using \code{Adam} with a learning rate of $10^{-3}$, a batch size of 16, and 50 epochs. Backpropagation across the quantum circuit is done via the parameter-shift algorithm \citep{schuld_evaluating_2019}. The \emph{sarcasm} task (binary) uses \code{BCEWithLogitsLoss}, while multi-class tasks (3- and 9-class) use \texttt{CrossEntropyLoss}; fixed task weights of $0.2$ balance gradients across tasks for our model, while we use a weight of $1.0$ for the classical and HQNN models. 
After each training epoch, we evaluate the model on the validation set. To checkpoint the best model, we use the best validation F1-score on the \emph{sarcasm} task, as this is the main task and is consistent with \citep{phukan_hybrid_2024}. We then report the final test metrics from this checkpoint. The metrics are as mentioned in Table \ref{tab:datasets_metrics}.




\end{appendices}


\bibliography{sn-bibliography}

\end{document}